\documentclass[acmsmall]{acmart}
\usepackage{bbm}
\usepackage{graphics}
\usepackage{subfigure}
\usepackage{algorithm}
\usepackage{algorithm}
\usepackage{algorithmicx}
\usepackage{algpseudocode}
\usepackage{fdsymbol}
\usepackage{multirow}
\usepackage{dutchcal}
\usepackage{cleveref}
\usepackage{hyperref}
\usepackage{caption}
\usepackage{wrapfig}
\usepackage{longtable}

\newcommand*{\Scale}[2][4]{\scalebox{#1}{$#2$}}
\theoremstyle{acmplain}
\newtheorem{theorem}{Theorem}[section]
\newtheorem{assumption}[theorem]{Assumption}

\renewcommand\footnotetextcopyrightpermission[1]{} % removes footnote with conference information in first column
\makeatletter
\renewcommand\@formatdoi[1]{\ignorespaces}
\makeatother
\DeclareMathOperator{\st}{s.t.}

\begin{document}

\title{Learning Individual Treatment Effects under Heterogeneous Interference in Networks}

\author{Ziyu Zhao}
\email{benzhao.styx@gmail.com}
\affiliation{%
 \institution{Zhejiang University}
 \city{HangZhou}
 \country{China}}
 
\author{Yuqi Bai}
\email{y78bai@uwaterloo.ca}
\affiliation{%
\institution{University of Waterloo}
\city{Waterloo}
\country{Canada}}
 
\author{Kun Kuang}
\email{kunkuang@zju.edu.cn}
\affiliation{%
 \institution{Zhejiang University}
 \city{HangZhou}
 \country{China}}

\author{Ruoxuan Xiong}
\email{ruoxuan.xiong@emory.edu}
\affiliation{%
 \institution{Emory University}
 \city{Atlanta}
 \country{USA}}
 
\author{Qingyu Cao}
\email{qingyu.cqy@alibaba-inc.com}
\affiliation{%
 \institution{Alibaba Group}
 \city{HangZhou}
 \country{China}}

\author{Fei Wu}
\email{wufei@cs.zju.edu.cn}
\affiliation{%
 \institution{Zhejiang University}
 \city{HangZhou}
 \country{China}}


\begin{abstract}
Estimating individual treatment effects in networked observational data is a crucial and increasingly recognized problem. One major challenge of this problem is violating the Stable Unit Treatment Value Assumption (SUTVA), which posits that a unit's outcome is independent of others' treatment assignments. However, in network data, a unit's outcome is influenced not only by its treatment (i.e., direct effect) but also by the treatments of others (i.e., spillover effect) since the presence of interference. Moreover, the interference from other units is always heterogeneous (e.g., friends with similar interests have a different influence than those with different interests). 
In this paper, we focus on the problem of estimating individual treatment effects (including direct effect and spillover effect) under heterogeneous interference in networks. To address this problem, we propose a novel Dual Weighting Regression (DWR) algorithm by simultaneously learning attention weights to capture the heterogeneous interference from neighbors and sample weights to eliminate the complex confounding bias in networks. We formulate the learning process as a bi-level optimization problem. Theoretically, we give a generalization error bound for the expected estimation error of the individual treatment effects. Extensive experiments on four benchmark datasets demonstrate that the proposed DWR algorithm outperforms the state-of-the-art methods in estimating individual treatment effects under heterogeneous network interference.
\end{abstract}



\keywords{Individual Treatment Effects, Spillover Effects, Heterogeneous Interference, Networked Data}

%%
%% This command processes the author and affiliation, and title
\maketitle
\pagestyle{plain}

% \begin{figure*}[htbp]
%     \centering
%     \subfigure[Heterogeneous in the social network]{
%         \begin{minipage}[b]{.25\linewidth}
%           \centering
%           \includegraphics[width=\linewidth]{hete_social_graph.pdf}
%           \label{fig:hete_social}
%         \end{minipage}
%   }\subfigure[Chain graph representation of data from a network of three units]{
%         \begin{minipage}[b]{.28\linewidth}
%           \centering
%           \includegraphics[width=\linewidth]{causal_graph.pdf}
%           \label{fig:causal_graph}
%         \end{minipage}
%   }\subfigure[Confounding Bias in the networked observational datasets]{
%         \begin{minipage}[b]{.28\linewidth}
%           \centering
%           \includegraphics[width=\linewidth]{interference.pdf}
%         \label{fig:interference}
%         \end{minipage}
%   }
%     \caption{The left figure shows a social network containing three units; each unit has different interests; the interference might be heterogeneous since the different social influences and social influences can be shown as similarities of interests. The middle figure shows the causal graph of three connected units and gives the Markov Blankets of $y_2$. The right figure shows that the main obstacle to treatment effect estimation is the spurious correlations between $\{\mathbf{X}, T, Z\}$, where $Z$ is a summary of the neighbors treatment}
%     \vspace{-0.15in}
%     \label{fig:causal}
% \end{figure*}

\section{Introduction}
With the surge in popularity of online social networks, there has been an exponential increase in the number of users, leading to the generation of vast quantities of observational data. This data is vital for estimating treatment effects in various fields, such as economics, epidemiology, and advertising.
Numerous methods \cite{johansson2016learning,louizos2017causal,yoon2018ganite,wager2018estimation,guo2020learning, chu2021graph, veitch2019using} have been proposed and achieved good results in some scenarios. However, the effectiveness of these methods relies on the \textit{stable unit treatment assumption} (SUTVA) \cite{cox1958planning}.
SUTVA assumes that the distribution of potential outcomes for one unit is not affected by the treatment assignment of other units when given the observed variables.
In social networks, however, \textit{interference} among individuals is a common occurrence. This interference is primarily attributed to social interactions, as discussed by \cite{forastiere2021identification}.
In epidemiology, for example, vaccination protects vaccinated individuals and reduces the probability of diagnosis in those around them \cite{nichol1995effectiveness}.
In econometric studies, neighborhood influence may also play a role in a household's decision to move \cite{sobel2006randomized}. 
In advertising, an ad's exposure may directly affect a user's purchase behavior and indirectly affect others in their social network through their acquisition behavior \cite{parshakov2020spillover}.
These examples show inter-unit interference, where one unit's treatment affects another's outcome.
In the presence of interference, a unit's outcome is determined not only by its treatment (i.e., direct effect) but also by the treatments of others (i.e., spillover effect), indicating the violation of SUTVA.
Hence, how to precisely estimate both direct and spillover effects from the networked observational data in the presence of interference is a vital and challenging problem. 

Previous literature on network interference \cite{liu2016inverse, sofrygin2017semi, ogburn2017causal, tchetgen2021auto} has primarily focused on estimating the average treatment effects (especially average spillover effect) in network observational data, lacking the ability to estimate the individual treatment effects.
Some recent approaches try to model the interference and use it to promote the performance of treatment effect estimation \cite{ma2021causal, ma2022learning}. However, these methods only consider the treatment of neighboring nodes as a feature to more accurately estimate the direct treatment effect, ignoring the spillover effect and failing to address the challenges encountered in estimating the spillover effect.
For literature studying the spillover effect \cite{li2020random,forastiere2021identification,jiang2022estimating}, anonymous interference or homogeneity is commonly assumed, implying no difference in the influence of neighboring nodes.
However, these assumptions do not necessarily hold in a real social network scenario since different units may respond differently to the treatments from other units, which means that the interference may be heterogeneous \cite{qu2021efficient}.
One of the primary sources of heterogeneity is the different social influences between connected units in social networks \cite{ma2011recommender, song2019session}.
Measuring heterogeneity is vital but overlooked for estimating treatment effects in networks.

\begin{figure*}[t]
    \centering
    \includegraphics[width=.6\linewidth]{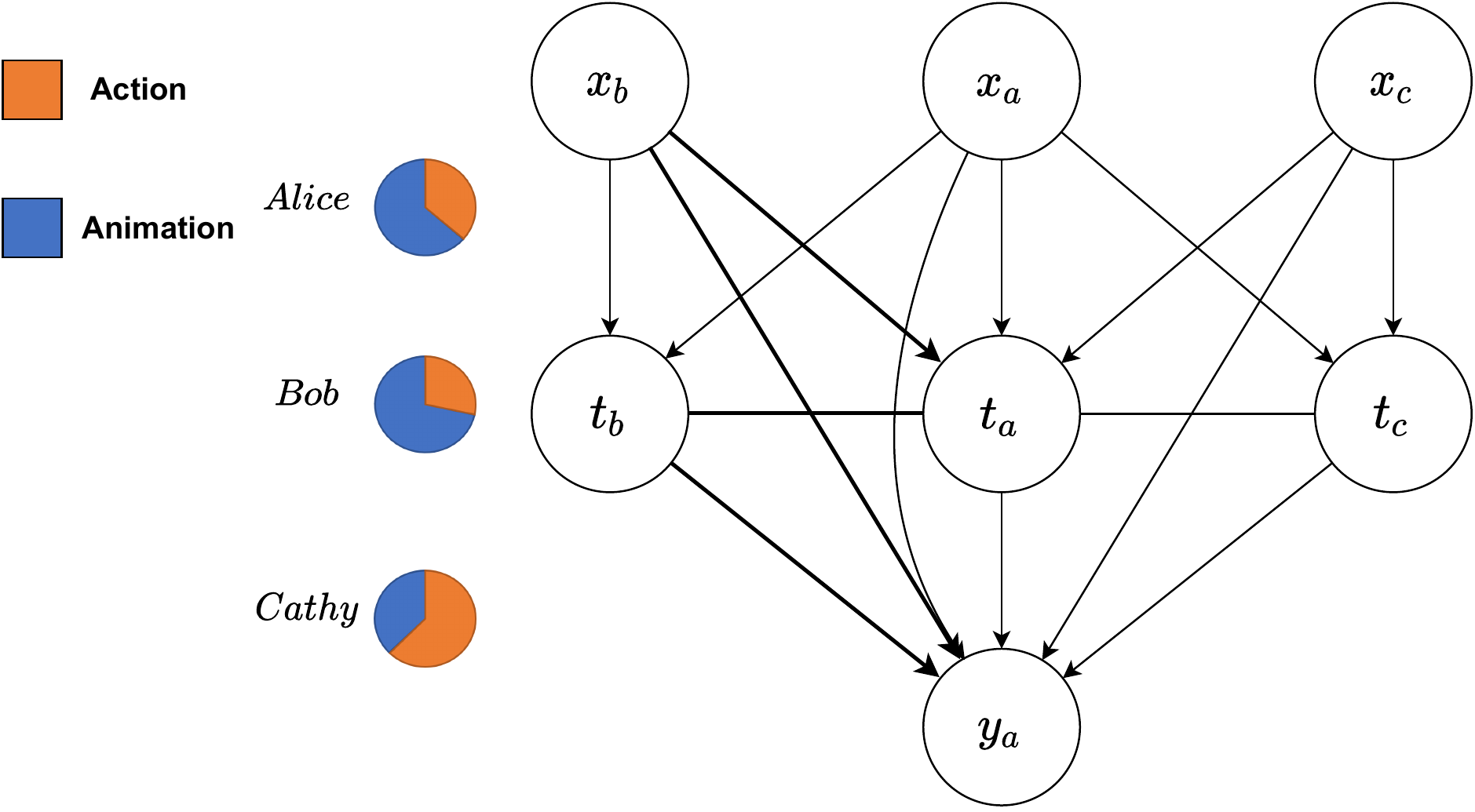}
    \vspace{-0.15in}
    \caption{A motivating example to illustrate the setting of heterogeneous interference in networks. The undirected edges mean that the treatment variables between connected units are associated, and the network data can be represented as a chain graph, which contains both directed and undirected edges \cite{tchetgen2021auto}.}
    \label{fig:example}
    \vspace{-0.15in}
\end{figure*}

\textbf{Motivating Example.} 
Fig.\ref{fig:example} presents a social network consisting of three units. We set the covariates $x$: the user's preferences for different categories of movies; the treatments $t$: to see \textit{James Bond} or \textit{Toy Story}; the outcomes $y$: the mood after watching the movie. 
As shown in Fig.\ref{fig:example}, we can see three types of interaction:
(i) a unit's preference for movies will affect the choice of movies (e.g., $x_a \rightarrow t_a$); (ii) a unit's preference will affect the choice of his/her friends (e.g., $x_a \rightarrow t_b$); (iii) a unit's choice of the movie will affect the choice of his/her friends (e.g., $t_a \rightarrow t_b$).
In this example, as for Alice, the confounding bias for estimating its direct (i.e., $t_a \rightarrow y_a$) and spillover (e.g., $\{t_b,t_c\} \rightarrow y_a$) effects is very complicated. 
Moreover, the interference from other units might be heterogeneous. Alice and Bob prefer action movies in this example, and Cathy is the opposite. Hence, Bob may have more influence on Alice than Cathy.
It is worth noting that, in contrast to the traditional causal graph modeled as a directed acyclic graph, the causal graph in Fig.\ref{fig:example} is a chain graph (a mixed graph containing both directed and undirected edges), similar to \cite{tchetgen2021auto}.
Here, we allow the treatment variables between units to interact, indicating that undirected edges exist between treatment variables of connected units.
% Such a setting is ubiquitous in real-life situations. For example, whether or not the neighboring nodes are vaccinated affects an individual's propensity to be vaccinated.

In this scenario, we confront two primary challenges in estimating individual treatment effects from network observational data in the presence of interference:

(i) \textbf{Heterogeneous interference.}
As our motivating example highlights, acknowledging the heterogeneity due to varying social influences is crucial in estimating treatment effects.
Yet, we observe that existing literature often overlooks this heterogeneity in network data. Studies focused on network interference \cite{aronow2017estimating, forastiere2021identification,jiang2022estimating} typically assume peer exposure as a uniform proportion of treated neighbors, neglecting the varied influences of different neighbors.
Similarly, research on networked observational data \cite{guo2020learning, chu2021graph} utilizes graph convolutional networks \cite{kipf2016semi} to aggregate neighbor node information to obtain node representations without considering differences in the social influence of neighbors.
This oversight often leads to inaccurate estimations of treatment effects in network settings, making the capture of interference heterogeneity a significant challenge.

(ii) \textbf{Complex Confounding Bias.}
In the context of networked observational data, the issue of confounding bias is exacerbated by interference.
As shown in Fig.\ref{fig:interference}, in the network scenario, confounding biases arise from the correlation between confounders $\{X,X_N\}$ (the covariates of a unit along with its neighbors' covariates), treatments $T$, and peer exposures $Z$ (the summary of neighborhood treatments).
When estimating the direct effect of treatment $T$ on outcome $Y$ ($T\rightarrow Y$), confounding bias arises from the covariates $\{X_N,X\}$ along with peer exposures $Z$. Similarly, in the estimation of the spillover effect of $Z$ on $Y$ ($Z \rightarrow Y$), the confounding bias is introduced by ${X_N, X}$ along with the treatment $T$. These correlations between covariates ${X_N, X}$, treatment $T$, and peer exposures $Z$ hinder the estimation of treatment effects.
Previous works \cite{johansson2016learning,li2020continuous,zou2020counterfactual,guo2020learning, chu2021graph, ma2021causal, ma2022learning}  have primarily addressed the correlation between confounders and treatments, falling short in such intricate scenarios.
Although \citet{jiang2022estimating, cristali2022using} attempt to address bias from interference, they overlook the association between $Z$ and $T$ and fail to model heterogeneous interference effectively.
%Eliminating the correlation between these three components is still a difficult challenge.

\begin{figure*}[t]
    \centering
    \includegraphics[width=.5\linewidth]{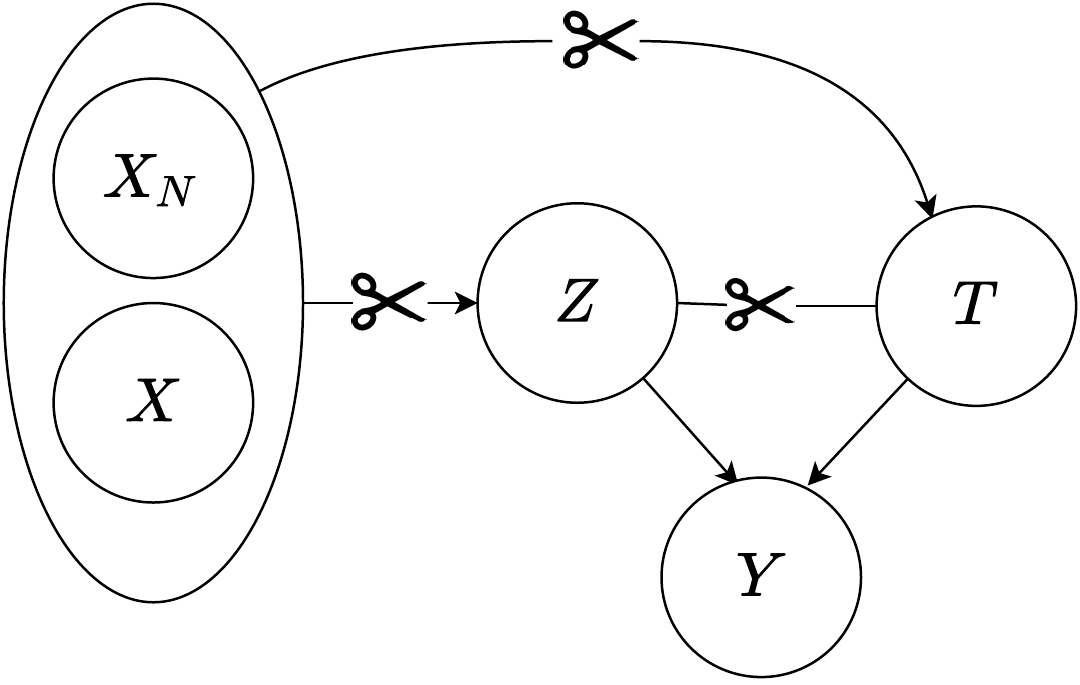}
    % \vspace{-0.15in}
    \caption{Confounding bias in the networked observational data. The correlations between $\{X,X_N\}$, $T$ and $Z$ leads to confounding bias. Therefore, to solve the problem of confounding bias, it is necessary to decorrelate these components.}
    \label{fig:interference}
    \vspace{-0.15in}
\end{figure*}

In this paper, we introduce a novel Dual Weighting Regression (DWR) algorithm, which simultaneously optimizes attention and sample weights to overcome the challenges previously outlined.
Specifically, the attention weights are designed to learn the heterogeneous interference from different nodes in a neighborhood through a graph attention mechanism.
With the attention weights, we summarize the neighboring nodes' treatment as the \textit{peer exposure} and aggregate the features of the neighboring nodes with Graph Attention Networks \cite{velivckovic2017graph,lee2019attention}.
On the other hand, the sample weights are designed to disentangle the associations between features, treatments, and peer exposures within networks. We create a calibration dataset where these elements are independent, specifically for training sample weights. These sample weights are then utilized in a weighted regression approach. 
The learning process of the DWR algorithm is formulated as a bi-level optimization problem by alternately optimizing the sample weight learning network and the outcome regression network.
Theoretically, We give a generalization-error bound for individual treatment effect estimation and show the effectiveness of the proposed DWR algorithm.
We compare our DWR algorithm with the state-of-the-art methods on several benchmark datasets. The empirical results show that the proposed algorithm outperforms these methods in both direct and spillover effects estimation.

Our contribution can be summarised as follows:
\begin{itemize}
    \item We investigate a more practical problem in estimating the individual treatment effects (e.g., direct and spillover effects) under heterogeneous interference in networks. 
    \item We propose a novel Dual Weighting Regression algorithm, which solves the heterogeneous and confounding bias challenges in the presence of interference by applying attention weights and sample weights to the regression.
    \item We theoretically give a generalization-error bound for treatment effects estimation and demonstrate the theoretical guarantees for our algorithm.
    \item The empirical results on four benchmark datasets show that the proposed Dual Weighting Regression algorithm outperforms the state-of-the-art methods.
\end{itemize}

\section{Problem Setup}
In this paper, we focus on estimating individual-level treatment effects from networked observational data in the presence of heterogeneous interference. 
% Further, we considered individual direct treatment effects as well as spillover effects.
Following \cite{guo2020learning}, the networked observational data can be formulated as  $\mathbbm{D} = \left(\{x_i, t_i, y_i \}_{i=1}^n, A \right)$.
For each unit $i$, we observe confounders $x_i \in \mathcal{X}$, binary treatment $t_i \in \{0,1\}$ and an outcome variable $y_i \in \mathbbm{R}$. $A$ is the adjacency matrix for an undirected graph $\mathbf{G}(\mathbf{V}, \mathbf{E})$, where $(v_i, v_j) \in \mathbf{E}$ indicates that there is an edge between node $v_i \in \mathbf{V}$ and $v_j \in \mathbf{V}$.
% $m_X$ denotes the dimension of the confounders. $A$ is the adjacency matrix for an undirected graph.
Following the Neyman-Rubin causal model, we posit the existence of potential outcomes for each unit $i$ under treatments $T$ is denoted by $y_i(T)$, where $T = [t_1,\cdots,t_n]$ is the treatment vector of all units.  
%We posit the following assumption as previous literature on network interference \cite{aronow2017estimating, forastiere2021identification}:
\begin{assumption}
\textbf{Network Interference}. 
\label{ass:neighborhood_interference}
Following \cite{aronow2017estimating, forastiere2021identification}, we assume that for any unit $i$,  $y_i(T) \!=\! y_i\!\left(t_i, G(T_{N(i)})\right)$, where $N(i)$ denote the neighborhood of unit $i$, and $T_{N(i)}$ denotes the collection of neighborhood treatments. The function $G: \{0,1\}^{|N(i)|} \rightarrow [0,1]$ is the exposure mapping function that summarizes the neighbors' treatment into a scalar.
\end{assumption}

We define the \textbf{peer exposure} as $z_i = G(T_{N(i)})$.
Previous work \cite{aronow2017estimating, forastiere2021identification} assume that the peer exposure is the proportion of the treated neighbors, that is, $z_i =  \frac{\sum_{j\in N(i)}t_j}{|N(i)|}$.
In this paper, we relax the assumption and assume that $z_i$ is a weighted sum of the neighbors' treatments:
\begin{eqnarray}
    z_i = \sum_{j\in N(i)}a_{ij}t_j \quad \st \sum_{j\in N(i)} a_{ij} = 1,
\end{eqnarray}
which means that the interference could be heterogeneous. 
Additionally, we posit that $a_{ij}$ can be portrayed by the social influence between connected units.

In this paper, we focus on estimating the individual treatment effect $\tau(x)$ given treatments $t_1$, $t_2$ and peer-exposure $z_1$, $z_2$ as follows:
\begin{eqnarray}
    \tau(x):= \mathbbm{E}[y(t_1,z_1) | x, X_{N(\cdot)}] - \mathbbm{E}[y(t_2,z_2) | x, X_{N(\cdot)}],
\end{eqnarray}
where $X_{N(\cdot)}$ denotes the collection of neighborhood covariates of a unit. The average treatment effect can be formulated as $ATE=\frac{1}{n}\sum_{i=1}^n \tau(x_i)$.
For simplicity, we use $\mathbf{X} \in \mathcal{X} \times \mathcal{X}^{|N|}$ to denote the covariates along with the neighbors' covariates, i.e., $\mathbf{X}_i:=\{x_i,X_{N(i)}\}$.
The causal effects can be identified with the following assumptions in the network scenario \cite{tchetgen2021auto}:
\begin{assumption}
\textbf{Markov property}.The outcome $y_i$ only depends on the covariates and treatment of the unit and its neighbors:
\begin{equation}
    y_i \Vbar x_j, t_j, y_j \mid \mathbf{X}_i, t_i, z_i, \quad j \in N(-i)
\end{equation}
where $N(-i)$ denotes the non-neighboring nodes of unit $i$.
\end{assumption}

\begin{assumption}
\textbf{Networked Unconfoundedness}. The treatment $t\in \mathcal{T}$, peer-exposure $z \in \mathcal{Z}$, potential outcomes $y(t,z)$ are independent given the joint covariates of a unit and its neighbors $\mathbf{X}$, i.e., $y(t,z) \Vbar \mathbf{X}$.
\end{assumption}

\begin{assumption}
\textbf{Overlap}. For any $\mathbf{X} \in \mathcal{X}\times\mathcal{X}^{N}$ such that $p(\mathbf{X})>0$, we have $0< p(t,z|\mathbf{X})<1$ for each $t\in \mathcal{T}$ and $z \in \mathcal{Z}$.
\end{assumption}

\begin{lemma}
\textbf{Identifiability of the treatment effect}. Under the assumption above, for any treatments $t_1$, $t_2$ and peer-exposure $z_1$, $z_2$, we have:
\begin{eqnarray}
    \mathbbm{E}[y(t_1,z_1)-y(t_2,z_2)|\mathbf{X}] = 
    \mathbbm{E}[y(t_1,z_1)|\mathbf{X}] - \mathbbm{E}[y(t_2,z_2)|\mathbf{X}].
\end{eqnarray}
which means that we can identify the treatment effect from the observational data.
\end{lemma}

In this paper we focus on estimating two specific treatment effects, which are individual level \textbf{direct treatment effect} $\mathbbm{E}[y(t=1,z)-y(t=0,z)|\mathbf{X}]$ and \textbf{spillover effect} $\mathbbm{E}[y(t=0,z)-y(t=0,z=0)|\mathbf{X}]$ from network observations.

\begin{figure*}[t]
    \centering
    \includegraphics[width=.8\linewidth]{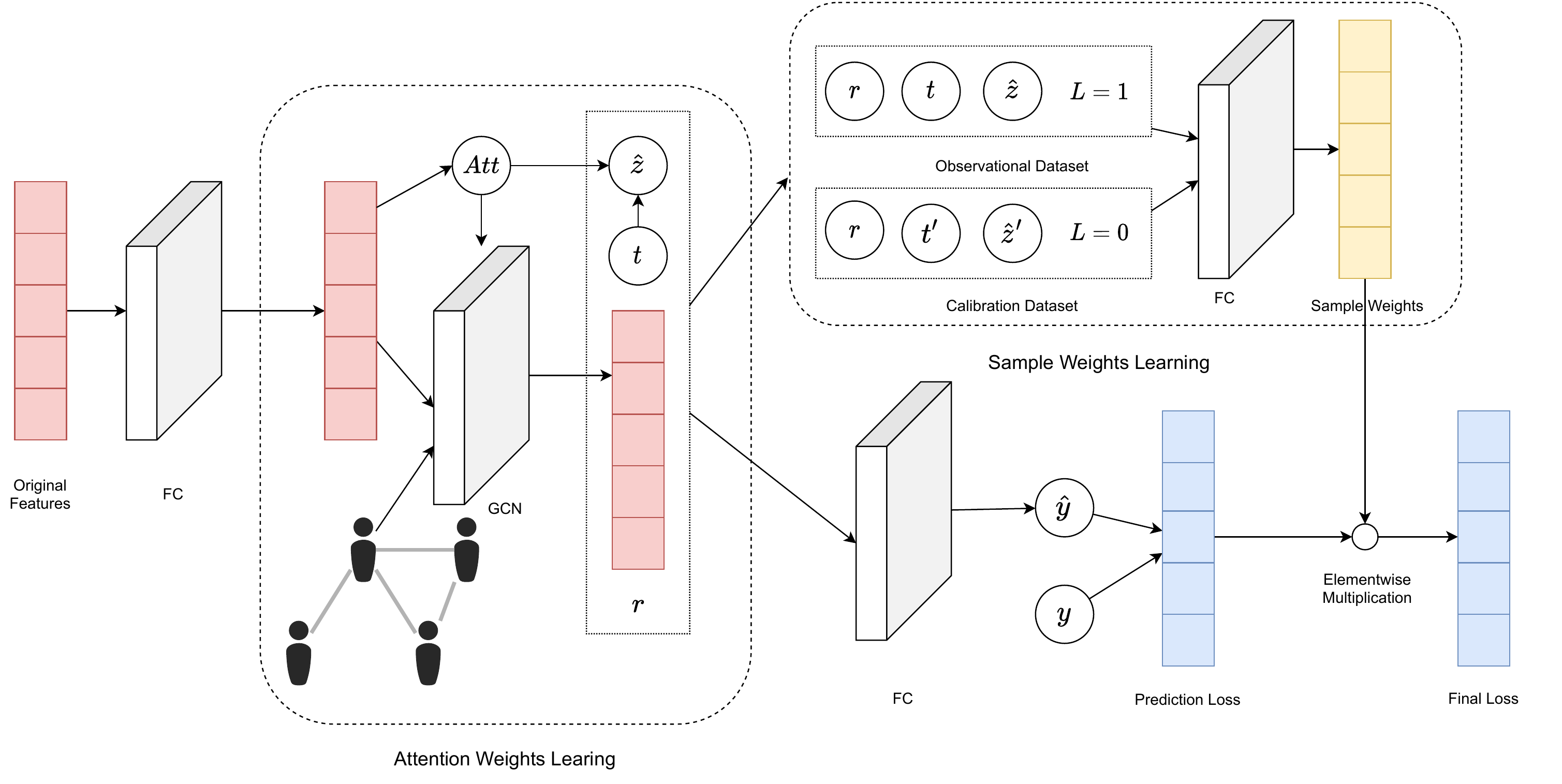}
    % \vspace{-0.15in}
    \caption{Overall framework of the proposed Dual Weighting Regression algorithm. The attention weights learning component is in section \ref{sec:att}. The sample weights learning component is in section \ref{sec:sr}. We formulate the entire learning process as a bi-level optimization problem. The detail of the bi-level optimization is in section \ref{sec:bi-level}.}
    \label{fig:model_structure}
    % \vspace{-0.15in}
\end{figure*}

\section{Dual Weighting Regression Algorithm}
In this section, we propose a novel algorithm named Dual Weighting Regression to solve the two challenges in learning individual treatment effects from networked observational data.

We first introduce the attention mechanism to capture the heterogeneous interference between individuals in the social network and use the learned attention weights to discriminately aggregate the neighboring information.
Then we elaborate on the counterfactual prediction problem in networked observational data in the presence of interference and propose a novel sample reweighting method to solve the confounding bias challenge.
To effectively coordinate attention weights and sample weights, we formulate the learning process as a bi-level optimization problem to alternately train the weighted regression model and sample weight learning model.
The overall algorithm is shown in Fig.\ref{fig:model_structure}. 

\subsection{Capturing Heterogeneous via Attention Weights}
\label{sec:att}
As previously mentioned, heterogeneity signifies the importance and influence of neighboring nodes. Therefore, discerning the varying significance of these neighboring nodes poses a substantial challenge. Drawing inspiration from studies that employed the attention mechanism in Graph Convolutional Networks \cite{lee2019attention}, we adopt the attention mechanism to effectively capture heterogeneity among nodes.
% Specifically, given a target object $i$, we assign importance weights $a_{i,j}$ to the unit $j$ in its neighborhood.
% The models based on GCN ignores the heterogeneous of neighbors and all the neighbors are treated equally, which leads to poor performance in counterfactual prediction. 
% To address these issues, we propose to model the social influence between connected units via attention mechanism and use Graph Attention Networks \cite{velivckovic2017graph} to aggregate the features.

Firstly, we apply fully connected layers to the original features $x_i$ to derive the individual representation $h_i$, which characterizes the user interests for each unit $i$.
With the learned individual representation $h_i$, we calculate the similarity between each unit $i$ and all its neighbors in the following way:
\begin{eqnarray}
    \label{eq:attention}
    a_{ij} = \frac{exp(h_i^T h_j)}{\sum_{k\in N(i)} exp(h_i^T h_k)}.
\end{eqnarray}
With the attention weights that measure the influence of friend $j$ on unit $i$, we can summarize the peer exposure as:
\begin{eqnarray}
\label{eq:peer_exposure}
    \hat{z}_i = \sum_{j\in N(i)} a_{ij} t_j .
\end{eqnarray}
At the same time, to resolve the problem that traditional GCN methods aggregate the features without distinguishing the different influences of the neighboring nodes, we utilize the attention weights to differentially aggregate the neighboring information.
With a self-connection edge that preserves a unit's interest, the representation of the unit $i$ can be eventually measured by:
\begin{eqnarray}
    \label{eq:representation}
    \phi(x_i,X_{N(i)}) = \sigma\left(\sum_{j\in N(i) \cup i} a_{ij} h_{j}\right)
\end{eqnarray}
where $\sigma(\cdot)$ denotes the activation function, and $\phi$ is the representation functions of the form $\phi : \mathcal{X} \times \mathcal{X}^{|N|} \rightarrow \mathcal{R}$. We use $r_i$ to denote the aggregated representation of unit $i$, $r_i:=\phi(x_i,X_{N(i)})$.

\subsection{Deconfounding via Sample Weights}
\label{sec:sr}
In this section, we present a method for learning sample weights to address the confounding bias in estimating individual treatment effects. This approach extends beyond traditional causal inference methods, often focusing on single-treatment scenarios or correlations between two components. Our method is developed to handle the complexities of simultaneously considering three components: representation $r$, treatment $t$, and peer exposure $z$. The essence of our method is the creation of a calibration distribution coupled with point-wise density ratio estimation. This integration is crucial as it ensures that our optimization objectives align the optimization objective with the calibration distribution's properties, enabling precise counterfactual prediction.

As shown in Fig.\ref{fig:interference}, to precisely estimate the individual treatment effects, we should eliminate the confounding bias by removing the correlation between the representation $r$, treatment $t$, and the peer exposure ${z}$. The key is to create a calibration distribution satisfying that $p(r,t,z)=p(r)p(t)p(z)$ while preserving the marginal distribution of each component. Inspired by the previous work \cite{li2020continuous}, we first construct a calibration dataset where $r$, $t$, and ${z}$ are independent of each other. 

The calibration datasets can be obtained by randomly shuffling the values of $t$ and $z$ in the observational dataset.
In the generated calibration dataset, the shuffled treatments $t^{\prime}$ and peer-exposures $z^{\prime}$ become independent with the representation $r$ and each other. Meanwhile, the permuted $t^{\prime}$ and $z^{\prime}$ preserve the marginal distribution of the original dataset. With the calibration dataset, the sample weights are learned as density ratio by comparing the observational distribution $\mathbf{D}_{obs}\{r, z, t\}$ and the calibration distribution $\mathbf{D}_{cal}\{r, z^{\prime}, t^{\prime} \}$.

We set the data points from the observational datasets $\mathbf{D}_{obs}$ as positive samples ($L=1$) and the data points from the calibration datasets $\mathbf{D}_{cal}$ as negative samples ($L=0$). By fitting a neural network $\pi$ to estimate $p(L|r, t, z)$, we can obtain the sample weights by:
\begin{eqnarray}
    \label{eq:one2one}
    W(r,t,z) = \frac{p(r)p(t)p(z)}{p(r,t,z)} 
    = \frac{p(r,t,z | L=0)}{p(r,t,z|L=1)} \\
    = \frac{p(L=1)}{p(L=0)}\cdot \frac{p(L=0|r,t,z)}{p(L=1|r,t,z)}
    = \frac{p(L=0|r,t,z)}{p(L=1|r,t,z)},
\end{eqnarray}
where $\frac{p(L=1)}{p(L=0)} = 1$ since the same sample size of the observational dataset and the calibration dataset.
The learned sample weights can be used to resolve the confounding bias problem in regressing the outcome.
$\pi$ is optimized in the following way:
\begin{eqnarray}
    \label{eq:loss_weights}
    \mathcal{L}_{\pi} = CrossEntropy(\pi(r_i,t_i,z_i),L_i)
\end{eqnarray}
The sample weight of unit $i$ can be obtained by:
\begin{eqnarray}
    \label{eq:sample_weights}
    W(r_i,t_i,z_i) = \frac{1-\pi(r_i,t_i,z_i)}{\pi(r_i,t_i,z_i)}.
\end{eqnarray}
With the learned weights, we try to regress the outcome by deep neural network $f: \mathcal{R}\times \mathcal{T} \times \mathcal{Z} \rightarrow \mathbbm{R}$ in the following way:
\begin{eqnarray}
    \label{eq:regression}
    \mathcal{L}_{f,\phi} = \frac{1}{n}\sum_{i=0}^{n}(f(r_i, z_i, t_i)-y_i)^2 W(r_i,t_i,z_i)
\end{eqnarray}
With the sample weights that resolve the confounding bias, we can accurately estimate the counterfactual outcome and the treatment effect from the neural network.

 \subsection{Bi-level Optimization}
\label{sec:bi-level}
% In the process of learning the sample weights, we use the $\{r_i, t_i, \hat{z}_i\}_{i=1}^n$ as the observational dataset, and randomly shuffle the treatment $t$ and the heterogeneous peer exposure $\hat{z}$ to generate the calibration dataset $\{r_i, t^{\prime}_i, \hat{z}^{\prime}_i\}_{i=1}^n$, and then fit the sample weight learning model $\pi_{\mu}={p}(L|r, t, \hat{z})$.
% Here we need to use dynamically changed $r$ and $\hat{z}$ at each step to update the sample weights.
% Moreover, to update the regression model, we need to dynamically learn the sample weights to remove the confounding bias.
% This raises the problem of "Chicken and Egg", where we fit the model $f$ and $\phi$ with the weights generated by network $\pi$, while at the same time, we need to fit the model $\pi$ with the datasets generated by $f$ and $\phi$. 
% Here, we formulate the learning process as a bi-level optimization problem:
% \begin{eqnarray}
%     % f_{\theta} &=& 
%     \mathop{\min}\limits_{f,\phi} \mathcal{L}_{f,\phi} \\
%     s.t. \quad \pi = \mathop{\min}\limits_{\pi} \mathcal{L_{\pi}}
% \end{eqnarray}
% The upper level minimizes the weighted loss of the outcome, and the lower level optimizes the sample weights that resolve the confounding bias. By alternately optimizing the sample weights learning network and the outcome regression network, we find an approximate solution for the original problem.
In the process of learning sample weights, we utilize the dataset $\{r_i, t_i, \hat{z}_i\}_{i=1}^n$ as our observational basis, while generating a calibration dataset $\{r_i, t^{\prime}_i, \hat{z}^{\prime}_i\}_{i=1}^n$ through random shuffling of treatment $t$ and heterogeneous peer exposure $\hat{z}$. To learn the sample weights, we fit the model $\pi_{\mu}={p}(L|r, t, \hat{z})$. This process requires dynamically updating $r$ and $\hat{z}$ at each step to refine the sample weights. Concurrently, we need to dynamically learn these sample weights to alleviate confounding bias in the regression model update.

This scenario presents a `Chicken and Egg' dilemma: the regression models $f$ and $\phi$ are fitted using weights generated by the network $\pi$, yet, simultaneously, the model $\pi$ depends on the datasets generated by $f$ and $\phi$. To address this complex interdependency, we frame the learning process as a bi-level optimization problem:
\begin{eqnarray}
    % f_{\theta} &=& 
    \mathop{\min}\limits_{f,\phi} \mathcal{L}_{f,\phi} \\
    s.t. \quad \pi = \mathop{\min}\limits_{\pi} \mathcal{L_{\pi}}
\end{eqnarray}
In this formulation, the upper level aims to minimize the weighted loss of the outcome, while the lower level focuses on optimizing the sample weights to eliminate confounding bias. However, due to the intertwined nature of these objectives, direct optimization is not straightforward. We tackle this challenge by iteratively optimizing the sample weight learning network and the outcome regression network, thereby seeking an approximate solution to the original intertwined problem.

The overall Dual Weighting Regression algorithm is summarized in Algorithm \ref{alg:weighted_regression}.

\renewcommand{\algorithmicrequire}{\textbf{Input:}}  % Use Input in the format of Algorithm
\renewcommand{\algorithmicensure}{\textbf{Output:}} % Use Output in the format of Algorithm
\begin{algorithm}[t]
    \caption{Dual Weighting Regression}
    \label{alg:weighted_regression}
    \begin{algorithmic}[1]
    \Require
        $x_1,\cdots,x_n$: Observed covariates; 
        $t_1,\cdots,t_n$: Treatment variable; 
        $y_1,\cdots,y_n$: Outcome; 
        $A$: Adjacency matrix; 
    \Ensure
        $f_\theta$: the regression model
    \Repeat
    \State
        Apply fully connected layers to the $x_i$ to obtain $h_i$.
    \State
        Obtain the attention score $a_{ij}$ between connected units $i$ and $j$ by eq.\ref{eq:attention}.
    \State
        Summarize the peer exposure $z_i$ by eq.\ref{eq:peer_exposure}.
    \State
        Obtain the representation $r_i$ as eq.\ref{eq:representation}.
    \State
        Generate sample weights by Algorithm \ref{alg:weight_learning}.
    \State
        Optimize the regression model by eq.\ref{eq:regression}.
    \Until{(Max Epoch Reach)}
    \end{algorithmic}
\end{algorithm}

\begin{algorithm}[t]
    \caption{Sample Weights Learning}
    \label{alg:weight_learning}
    \begin{algorithmic}[1]
    \Require
        $\{r,t,\hat{z}\}$: observational dataset;
        $\pi$: sample weights learning model
    \Ensure
        $W(r,t,\hat{z})$: sample weights
    \Repeat
    \State
        Random shuffle $t$ and $\hat{z}$ to obtain the calibration dataset $\{r, t^{\prime}, \hat{z}^{\prime}\}$.
    \State
        Label the observational dataset as positive samples ($L=1$) and the calibration dataset as negative samples ($L=0$).
    \State
        Train the binary classification model $\pi$ with eq.\ref{eq:loss_weights}. 
    \Until{(Max Epoch Reach)}
    \State Obtain the sample weights $W(r,t,\hat{z})$ by eq.$\ref{eq:sample_weights}$.
    \State \Return $W(r,t,\hat{z})$
    \end{algorithmic}
\end{algorithm}

\section{Theoretical results}
In this section, we upper bound the individual treatment effects estimation error based on the weighted factual error term and the distribution distance between the weighted observational distribution and the fully random distribution.

For simplicity, we use $\mathbf{X} \in \mathcal{X} \times \mathcal{X}^{|N|}$ to denote the covariates along with the neighbor's covariates.
At the same time, following the previous literature on representation learning \cite{shalit2017estimating,bellot2022generalization}, we assume that the representation function $\phi$ is a twice differentiable, one-to-one function.

Then the loss function can be defined as follows:
\begin{equation}
    l_{f,\phi}(\mathbf{X}, t, z) = \int_\mathcal{Y} L(y(t, z), f(r, t, z))p(y(t, z) | \mathbf{X})dy(t,z)
\end{equation}
We define the expected factual and counterfactual losses of $f$ at treatment $t \in \mathcal{T}$ and $z \in \mathcal{Z}$ as:
\begin{eqnarray}
    \epsilon_F(t,z) = \mathbbm{E}_{\mathbf{X} \sim p(\mathbf{X}|t,z)}[l_{f,\phi}(\mathbf{X}, t, z)]
\end{eqnarray}
\begin{equation}
    \epsilon_{CF}(t,z) = \mathbbm{E}_{\mathbf{X} \sim p(\mathbf{X})}[l_{f,\phi}(\mathbf{X}, t, z)]
\end{equation}
and the expected factual and counterfactual errors in the following way:
\begin{equation}
\epsilon_F = \mathbbm{E}_{\mathbf{X},t,z \sim p(\mathbf{X},t,z)}[l_{f,\phi}(\mathbf{X}, t, z)]
\end{equation}
\begin{equation}
\epsilon_{CF} = \mathbbm{E}_{\mathbf{X},t,z \sim p(\mathbf{X})p(t)p(z)}[l_{f,\phi}(\mathbf{X}, t, z)]
\end{equation}

The weighted factual error can be defined as:
\begin{equation}
    \epsilon_F^w = \mathbbm{E}_{\mathbf{X},t,z \sim p(\mathbf{X},t,z)}[l_{f,\phi}(\mathbf{X}, t, z)W(r, t, z)]
\end{equation}

Following the theoretical analysis of \cite{zou2020counterfactual}, we have an upper bound of $\epsilon_{CF}$ based on $\epsilon_F^w$ and Integral Probability Metric (IPM), which measures the distance between two distributions. More specific, for two distribution on $\mathcal{X} \times \mathcal{X}^{N} \times \mathcal{T} \times \mathcal{Z}$, $p_1(\mathbf{X},T,Z)$ and $p_2(\mathbf{X},T,Z)$, and a family $G$ of functions $g: \mathcal{X} \times \mathcal{X}^{N} \times \mathcal{T} \times \mathcal{Z} \rightarrow \mathbbm{R}$, we have:
\begin{eqnarray}
   \nonumber &&IPM_G(p_1(\mathbf{X},t,z), p_2(\mathbf{X},t,z)) \\
   \nonumber && = sup_{g\in G}|\int_{\mathbf{X}}\int_T\int_Z (p_1(\mathbf{X},t,z) - p_2(\mathbf{X},t,z))g(\mathbf{X},T,Z)d \mathbf{X} dT dZ|.
\end{eqnarray}
\begin{theorem}
\label{thm:cf_upper}
Assuming a family $G$ of functions $g$: $\mathcal{X} \times \mathcal{X}^{N} \times \mathcal{T} \times \mathcal{Z} \rightarrow \mathbbm{R}$, the loss function $\mathcal{L}(f(r,t,z),y(\mathbf{X},t,z))\in G$, we have:
\begin{eqnarray}
    \epsilon_{CF} \leq \epsilon_F^w + IPM_G(W(r,T,Z)p(\mathbf{X},t,z), p(\mathbf{X})p(t)p(z)).
\end{eqnarray}
\end{theorem}
% With $W(r,t,z) = p(\mathbf{X})p(t)p(z)) / p(\mathbf{X},t,z)$, we have $IPM_G = 0$ and $\epsilon_{CF} = \epsilon_F^w$.
% The proof can be found in the appendix.
\begin{equation}
    \epsilon_{CF} \leq \epsilon_F^w + IPM_G(W(\mathbf{X},T,Z)p(\mathbf{X},t,z), p(\mathbf{X})p(t)p(z)).
\end{equation}
\begin{proof}
\begin{eqnarray}
\nonumber && \epsilon_{CF} - \epsilon_{F}^w \\
\nonumber  &=&\int_{\mathbf{X}} \int_{{T}} \int_{{Z}}(p(\mathbf{X})p(t)p(z) - W(\mathbf{X},T,Z) p(\mathbf{X},T,Z))\mathcal{L}(f(\mathbf{X},T,Z),y(\mathbf{X},T,Z) \\
\nonumber  &\leq& | \int_{\mathbf{X}} \int_{{T}} \int_{{Z}}(p(\mathbf{X})p(T)p(Z) - W(\mathbf{X},T,Z) p(\mathbf{X},T,Z))\mathcal{L}(f(\mathbf{X},T,Z),y(\mathbf{X},T,Z) | \\
\nonumber  &\leq& \mathop{\sup}\limits_{g \in G} | \int_{\mathbf{X}} \int_{{T}} \int_{{Z}}(p(\mathbf{X})p(T)p(Z) - W(\mathbf{X},T,Z) p(\mathbf{X},T,Z))g(\mathbf{X},T,Z)| \\
\nonumber &=& IPM_G(W(\mathbf{X},T,Z)p(\mathbf{X},T,Z), p(\mathbf{X})p(T)p(Z))
\end{eqnarray}
when $W(\mathbf{X},T,Z)=p(\mathbf{X})p(T)p(Z)/p(\mathbf{X},T,Z)$ we have:
\begin{eqnarray}
\nonumber && IPM_G(W(\mathbf{X},T,Z)p(\mathbf{X},T,Z), p(\mathbf{X})p(T)p(Z)) \\
\nonumber  &=& \mathop{\sup}\limits_{g \in G} | \int_{\mathbf{X}} \int_{{T}} \int_{{Z}}(p(\mathbf{X})p(T)p(Z) - \frac{p(\mathbf{X})p(T)p(Z)}{p(\mathbf{X},T,Z)} p(\mathbf{X},T,Z))g(\mathbf{X},T,Z)|\\
\nonumber &=& 0
\end{eqnarray}
and $\epsilon_{CF} = \epsilon_{F}^w$.
\end{proof}

From theorem.\ref{thm:cf_upper}, we know that the counterfactual error can be upper bounded by the weighted factual error with an IPM term. 
Without controlling the confounding bias, direct regression on the outcome cannot guarantee the performance on counterfactual estimation since the $IPM_G$ term could be extremely high.
% With the sample weight $W(r,t,z) = \frac{p(\mathbf{X})p(t)p(z)}{ p(\mathbf{X},t,z)}$, we can see that minimizing the $\epsilon_{F}^w$ equals to optimize the $\epsilon_{CF}$.

\begin{definition}
\textbf{Treatment effect and error for selected treatments $t_1$, $t_2$ and selected peer-exposure $z_1$, $z_2$}. We define the treatment effect between treatments $t_1$, $t_2$ and peer-exposure $z_1$, $z_2$ as:
\begin{eqnarray}
    \tau_{(t_1,z_1),(t_2,z_2)}(\mathbf{X})= m(\mathbf{X},t_1,z_1) - m(\mathbf{X},t_2,z_2),
\end{eqnarray}
where $m(\mathbf{X},t,z):=\mathbbm{E}[y(t,z)|\mathbf{X}]$. 

We define its estimate given a prediction function $f$ and representation function $\phi$ by,
\begin{eqnarray}
    \hat{\tau}_{(t_1,z_1),(t_2,z_2)}(\mathbf{X})= f(r,t_1,z_1) - f(r,t_2,z_2).
\end{eqnarray}
The error in treatment effect estimation is defined as:
\begin{eqnarray}
    \nonumber \epsilon_{(t_1,z_1),(t_2,z_2)}(f,\phi) := 
    \int_{\mathbf{X}} \left(\hat{\tau}_{(t_1,z_1),(t_2,z_2)}(\mathbf{X}) - \tau_{(t_1,z_1),(t_2,z_2)}(\mathbf{X})\right)^2 p(\mathbf{X})d\mathbf{X}
\end{eqnarray}
\end{definition}

\begin{lemma}
\label{lemma:1}
The following derivation holds:
\begin{eqnarray}
    \nonumber && \epsilon_{CF}(t,z) = \int_{\mathbf{X}} l_f(\mathbf{X},t,z) p(\mathbf{X}) d\mathbf{X} \\
    \nonumber &=& \int_{\mathbf{X}}\int_{\mathcal{Y}} (y(t,z)-f(\mathbf{X},t,z))^2 p(y(t, z) | \mathbf{X})p(\mathbf{X})dy(t,z)d\mathbf{X} \\
    \nonumber &=& \int_{\mathbf{X}}\int_{\mathcal{Y}} (f(\mathbf{X},t,z)-m(\mathbf{X},t,z))^2 p(y(t, z) | \mathbf{X})p(\mathbf{X})dy(t,z)d\mathbf{X} \\
    \nonumber &+& \int_{\mathbf{X}}\int_{\mathcal{Y}} (m(\mathbf{X},t,z)-y(t,z))^2 p(y(t, z) | \mathbf{X})p(\mathbf{X})dy(t,z)d\mathbf{X} \\
    \nonumber &+& 2 \int_{\mathbf{X}}\int_{\mathcal{Y}} (m(\mathbf{X},t,z)-y(t,z))(f(\mathbf{X},t,z)-m(\mathbf{X},t,z)) p(y(t, z) | \mathbf{X})p(\mathbf{X})dy(t,z)d\mathbf{X} \\
    \nonumber &=& \int_{\mathbf{X}}\int_{\mathcal{Y}} (f(\mathbf{X},t,z)-m(\mathbf{X},t,z))^2 p(y(t, z) | \mathbf{X})p(\mathbf{X})dy(t,z)d\mathbf{X} + \sigma_{Y(t,z)}
\end{eqnarray}
\end{lemma}

\begin{theorem}
\label{thm:effect_upper}
The generalization bound in treatment effect estimation for selected treatments $t_1$, $t_2$ and selected peer-exposure $z_1$, $z_2$:
\begin{eqnarray}
    \nonumber &&\epsilon_{(t_1,z_1),(t_2,z_2)}(f,\phi) \\
    \nonumber &&\leq \Scale[0.9]{2\epsilon_F^w(t_1,z_1) +  2 IPM_G \left(W(r,t_1,z_1)p(\mathbf{X},t_1,z_1), p(\mathbf{X})p(t_1)p(z_1) \right)} \\
    \nonumber &&+ \Scale[0.9]{2\epsilon_F^w(t_2,z_2) +  2 IPM_G \left(W(r,t_2,z_2)p(\mathbf{X},t_2,z_2), p(\mathbf{X})p(t_2)p(z_2) \right)} \\
    \nonumber &&- \Scale[0.9]{2\sigma_{Y(t_1,z_1)}(p(\mathbf{X})) - 2\sigma_{Y(t_2,z_2)}(p\mathbf{X})}
\end{eqnarray}
where $\sigma_{Y(t,z)}(p(\mathbf{X}))$ is the variance of the random variables $Y(t,z)$ under distribution $p(\mathbf{X})$.
\end{theorem}

\begin{proof}
We define that $m(\mathbf{X},t,z) := \mathbbm{E}[y(t,z)|\mathbf{X}]$.

\begin{eqnarray}
    \nonumber &&\epsilon_{(t_1,z_1),(t_2,z_2)}(f) \\
    \nonumber &=& \Scale[0.9]{\int_{\mathbf{X}} (f(t_1,z_1,\mathbf{X})-m(t_1,z_1,\mathbf{X})+ f(t_2,z_2,\mathbf{X})-m(t_2,z_2,\mathbf{X}))p(\mathbf{X}))^2 d\mathbf{X}}\\
    \nonumber &\leq& \Scale[0.9]{2\int_{\mathbf{X}} (f(t_1,z_1,\mathbf{X})-m(t_1,z_1,\mathbf{X})) p(\mathbf{X}))^2 d\mathbf{X}} \\
    \nonumber &+& \Scale[0.9]{2\int_{\mathbf{X}}(f(t_2,z_2,\mathbf{X})-m(t_2,z_2,\mathbf{X}))p(\mathbf{X}))^2 d\mathbf{X}} \\
    \nonumber &=& 2(\epsilon_{CF}(t_1,z_1)-\sigma_{Y(t_1,z_1)}) + 2(\epsilon_{CF}(t_2,z_2)-\sigma_{Y(t_2,z_2)})\\
    \nonumber &\leq& 2\epsilon_F^w(t_1,z_1) +  2 IPM_G(W(\mathbf{X},t_1,z_1)p(\mathbf{X},t_1,z_1), p(\mathbf{X})p(t_1)p(z_1)) \\
    \nonumber &+& 2\epsilon_F^w(t_2,z_2) +  2 IPM_G(W(\mathbf{X},t_2,z_2)p(\mathbf{X},t_2,z_2), p(\mathbf{X})p(t_2)p(z_2)) \\
    \nonumber &-& 2\sigma_{Y(t_1,z_1)} - 2\sigma_{Y(t_2,z_2)}
\end{eqnarray}
\end{proof}
The first inequality holds since $(x+y)^2 \leq 2(x^2+y^2)$, the second equality holds by Lemma \ref{lemma:1}, and the second inequality holds by Theorem 1.
The error for individual treatment effect estimation is bounded by the weighted factual error term and the IPM between the weighted observational distribution and fully random distribution. 
The learned sample weights in the DWR algorithm satisfy that $W(r,t,z) \approx \frac{p(r)p(t)p(z)}{p(r,t,z)}$, with the assumption that the representation function is a 
twice differentiable, one-to-one function, we have $W(r,t,z) \approx \frac{p(\mathbf{X})p(t)p(z)}{p(\mathbf{X},t,z)}$.
With confounding bias successfully removed ($IPM \approx 0$), the proposed method can well bound the treatment effects estimation error by optimizing $\epsilon_F^w$.

% \begin{table}
%   \caption{Real World Social Networks}
%   \label{tab:graphs}
%   \begin{tabular}{ccc}
%     \toprule
%     Dataset Name & Instance & Edges  \\
%     \midrule
%     BlogCatalog & 5,196 & 173,468  \\
%     Flickr & 7,575 & 239,738  \\
%     Hamsterster & 2,426 & 16,630  \\
%     Fb-pages-tvshow & 3,892 & 17,262  \\
%   \bottomrule
% \end{tabular}
% \vspace{-0.15in}
% \end{table}
 
\begin{table*}
\centering
\caption{Results on both direct and spillover effects estimation. We report the $\sqrt{\epsilon_{PEHE}^{DE}}$, $\epsilon_{MAE}^{DE}$, $\sqrt{\epsilon_{PEHE}^{SE}}$, and $\epsilon_{MAE}^{SE}$. The best-performing method is bolded. The "N/A" represents that the method does not consider the spillover effect.}
\label{tab:main_res}
\scalebox{0.9}
{
\resizebox{\linewidth}{!}
{
\begin{tabular}{ccccccccc}
    \hline
        \multicolumn{9}{c}{Flickr} \\ \hline
        & \multicolumn{2}{c}{$\sqrt{\epsilon_{PEHE}^{DE}}$}  &\multicolumn{2}{c}{$\epsilon_{MAE}^{DE}$} & \multicolumn{2}{c}{$\sqrt{\epsilon_{PEHE}^{SE}}$} & \multicolumn{2}{c}{$\epsilon_{MAE}^{SE}$} \\ \hline
        ~ & within-sample & out-of-sample & within-sample & out-of-sample & within-sample & out-of-sample & within-sample & out-of-sample \\ \hline
        CFR & 5.60(0.43) & 5.53(0.44) & 3.59(0.26) & 3.51(0.32) & N/A & N/A & N/A & N/A \\
        CFR+$\bar{z}$ & 5.72(0.33) & 5.81(0.38) & 3.70(0.28) & 3.75(0.35) & 8.10(0.35) & 8.35(0.41) & 1.76(0.15) & 1.85(0.20)  \\ \hline
        TARNET & 5.79(0.54) & 5.86(0.62) & 3.78(0.40) & 3.85(0.49) & N/A & N/A & N/A & N/A \\
        TARNET+$\bar{z}$ & 5.40(0.53) & 5.51(0.65) & 3.46(0.35) & 3.59(0.45) & 8.11(0.33) & 8.30(0.32) & 1.76(0.15) & 1.84(0.21) \\ \hline
        ND & 6.52(0.37) & 6.33(0.35) & 2.33(0.33) & 2.08(0.27) & N/A & N/A & N/A & N/A \\
        ND+$\bar{z}$ & 6.57(0.33) & 6.54(0.32) & 1.55(0.33) & 1.63(0.41) & 8.30(0.41) & 8.20(0.66) & 2.98(0.41) & 3.03(0.66)  \\ \hline
        GIAL & 5.32(0.51) & 5.32(0.51) & 1.45(0.31) & 1.45(0.31) & N/A & N/A & N/A & N/A \\
        GIAL+$\bar{z}$ & 5.19(0.58) & 5.19(0.58) & 0.94(0.37) & 0.94(0.37) & 6.54(0.63) & 6.58(0.79) & 2.43(0.54) & 2.42(0.59)  \\ \hline
        GPS & 8.81(0.44) & 8.90(0.72) & 3.66(0.64) & 3.62(0.66) & 8.81(0.44) & 8.90(0.72) & 1.92(0.17) & 1.91(0.25) \\
        PCN & 5.24(0.67) & 5.23(0.73) & 2.51(0.36) & 2.48(0.45) & 8.10(0.343) & 8.02(0.26) & 2.09(0.18) & 2.04(0.33) \\ 
        GCN$+\hat{H}SIC$ & 5.17(0.54) & 5.18(0.49) & 2.43(0.27) & 2.38(0.25) & 8.05(0.24) & 7.83(0.47) & 2.20(0.43) & 2.18(0.60) \\ 
        NetEst & 4.31(0.63) & 4.39(0.74) & 2.20(0.51) & 2.25(0.58) & 7.44(0.36) & 7.67(0.37) & 1.13(0.18) & 1.24(0.28) \\  
        \hline
        DWR & \textbf{2.79(0.21)} & \textbf{2.75(0.30)} & \textbf{0.39(0.24)} & \textbf{0.38(0.26)} & \textbf{3.38(0.36)} & \textbf{3.37(0.42)} & 0.66(0.15) & 0.64(0.19) \\ 
        DWR w/o w & 2.94(0.53) & 2.99(0.52) & 0.80(0.35) & 0.78(0.36) & 3.38(0.60) & 3.46(0.80) & 0.98(0.41) & 0.99(0.47)\\ 
        DWR w/o att & 4.99(1.01) & 4.87(0.97) & 0.82(0.92) & 0.71(0.85) & 4.92(0.72) & 4.96(0.77) & 0.74(0.49) & 0.76(0.51)  \\
        DWR w/o att \& w & 4.73(0.52) & 4.69(0.54) & 0.80(0.54) & 0.74(0.52) & 4.73(0.52) & 4.76(0.50) & \textbf{0.56(0.42)} & \textbf{0.51(0.43)} \\
        \hline
        \hline
        \multicolumn{9}{c}{BlogCatalog} \\ \hline
        & \multicolumn{2}{c}{$\sqrt{\epsilon_{PEHE}^{DE}}$}  &\multicolumn{2}{c}{$\epsilon_{MAE}^{DE}$} & \multicolumn{2}{c}{$\sqrt{\epsilon_{PEHE}^{SE}}$} & \multicolumn{2}{c}{$\epsilon_{MAE}^{SE}$} \\ \hline
        ~ & within-sample & out-of-sample & within-sample & out-of-sample & within-sample & out-of-sample & within-sample & out-of-sample\\ \hline
        CFR & 5.25(0.54) & 5.28(0.47) & 3.58(0.40) & 3.60(0.30) & N/A & N/A & N/A & N/A \\
        CFR+$\bar{z}$ & 5.25(0.42) & 5.08(0.39) & 3.56(0.32) & 3.44(0.27) & 8.04(0.23) & 7.77(0.30) & 1.93(0.13) & 1.81(0.24) \\ \hline
        TARNET & 4.77(0.51) & 4.85(0.51) & 3.17(0.35) & 3.24(0.38) & N/A & N/A & N/A & N/A \\
        TARNET+$\bar{z}$ & 5.12(0.62) & 5.08(0.44) & 3.45(0.43) & 3.42(0.30) & 8.03(0.26) & 7.98(0.39) & 1.91(0.11) & 1.93(0.35) \\ \hline
        ND & 6.29(0.35) & 6.22(0.34) & 2.95(0.74) & 2.93(0.61) & N/A & N/A & N/A & N/A \\
        ND+$\bar{z}$ & 6.21(0.37) & 6.21(0.46) & 2.78(0.63) & 2.74(0.61) & 8.16(0.26) & 8.14(0.29) & 2.43(0.39) & 2.44(0.43) \\ \hline
        GIAL & 3.68(0.40) & 3.68(0.40) & 1.20(0.21) & 1.20(0.21) & N/A & N/A & N/A & N/A \\
        GIAL+$\bar{z}$ & 3.67(0.44) & 3.67(0.44) & 1.16(0.28) & 1.16(0.28) & 7.99(0.50) & 7.85(0.60)& 1.73(0.29) & 1.62(0.53) \\ \hline
        GPS & 10.69(0.55) & 10.89(0.35) & 3.88(0.25) & 3.93(0.24) & 10.69(0.55) & 10.89(0.35) & 2.28(0.21) & 2.31(0.34) \\
        PCN & 6.09(0.51) & 6.00(0.51) & 2.64(0.32) & 2.37(0.41) & 8.29(0.34) & 7.99(0.47) & 3.18(0.89) & 1.99(0.30) \\ 
        GCN$+\hat{H}SIC$ & 5.24(0.67) & 5.23(0.73) & 2.51(0.36) & 2.48(0.45) & 8.10(0.343) & 8.02(0.26) & 2.09(0.18) & 2.04(0.33) \\ 
        NetEst & 3.41(0.48) & 3.44(0.53) & 1.82(0.37) & 1.85(0.37) & 7.75(0.32) & 7.68(0.33) & 1.85(0.21) & 1.71(0.23) \\ \hline
        DWR & \textbf{2.62(0.41)} & \textbf{2.54(0.34)} & \textbf{0.46(0.36)} & \textbf{0.39(0.33)} & 4.85(0.50) & 4.83(0.55) & \textbf{1.12(0.37)} & \textbf{1.09(0.42)} \\
        DWR w/o w & 2.71(0.55) & 2.64(0.61) & 0.89(0.42) & 0.86(0.46) & \textbf{4.54(0.48)} & \textbf{4.65(0.48)} & 1.21(0.36) & 1.24(0.33) \\ 
        DWR w/o att & 3.36(0.70) & 3.34(0.76) & 0.53(0.54) & 0.51(0.60) & 5.86(0.72) & 5.86(0.79) & 1.14(0.64) & 1.25(0.66) \\ 
        DWR w/o att \& w & 3.16(0.57) & 3.13(0.62) & 0.54(0.49) & 0.53(0.50) & 5.55(0.72) & 5.51(0.87) & 1.15(0.60) & 1.27(0.65) \\ 
        \hline \hline
        \multicolumn{9}{c}{Hamsterster} \\ \hline
        & \multicolumn{2}{c}{$\sqrt{\epsilon_{PEHE}^{DE}}$}  &\multicolumn{2}{c}{$\epsilon_{MAE}^{DE}$} & \multicolumn{2}{c}{$\sqrt{\epsilon_{PEHE}^{SE}}$} & \multicolumn{2}{c}{$\epsilon_{MAE}^{SE}$} \\ \hline
        ~ & within-sample & out-of-sample & within-sample & out-of-sample & within-sample & out-of-sample & within-sample & out-of-sample\\ \hline
        CFR & 6.40(0.41) & 6.31(0.41) & 2.39(0.61) & 2.29(0.80) & N/A & N/A & N/A & N/A \\
        CFR+$\bar{z}$ & 6.22(0.38) & 6.29(0.53) & 2.14(0.56) & 2.15(0.62) & 7.50(0.41) & 7.67(0.81) & 1.70(0.17) & 1.88(0.57) \\ \hline
        TARNET & 6.73(1.12) & 6.54(0.90) & 2.29(0.63) & 2.08(0.82) & N/A & N/A & N/A & N/A \\
        TARNET+$\bar{z}$ & 6.38(0.29) & 6.30(0.47) & 2.30(0.61) & 2.34(0.56) & 7.48(0.49) & 7.40(0.77) & 1.74(0.20) & 1.78(0.35) \\ \hline
        ND & 8.71(0.41) & 8.90(0.44) & 2.05(0.90) & 1.85(1.11) & N/A & N/A & N/A & N/A \\
        ND+$\bar{z}$ & 8.68(0.49) & 8.94(0.43) & 3.05(0.52) & 2.88(0.98) & 7.49(0.46) & 7.23(0.99) & 1.74(0.35) & 1.62(0.63) \\ \hline
        GIAL & 7.10(0.40) & 7.10(0.40) & 0.96(0.41) & 0.96(0.41) & N/A & N/A & N/A & N/A \\
        GIAL+$\bar{z}$ & 6.87(0.36) & 6.87(0.36) & 0.71(0.65) & 0.71(0.65) & 4.89(0.44) & 5.36(0.80) & 0.68(0.37) & 0.76(0.45) \\ \hline
        GPS &  7.78(0.33) & 8.20(0.90) & 0.56(2.43) & 0.67(2.36) & 7.78(0.33) & 8.20(0.90) & 0.31(0.93) & 0.33(0.61) \\
        GCN$+\hat{H}SIC$ & 6.24(0.32) & 6.11(0.33) & 2.23(0.48) & 2.20(0.53) & 9.36(0.87) & 9.46(0.83) & 0.55(1.29) & 0.53(1.26) \\ 
        PCN & 4.49(1.04) & 4.66(1.05) & 3.12(1.15) & 3.41(1.14) & 4.48(1.04) & 4.67(1.05) & 1.87(0.06) & 1.88(0.28) \\ 
        NetEst & 9.01(0.37) & 9.21(0.62) & 1.69(0.75) & 1.76(0.99) & 6.36(0.27) & 6.96(0.79) & 0.34(0.27) & 0.39(0.41) \\ \hline 
        DWR & \textbf{3.71(0.19)} & \textbf{3.75(0.26)} & \textbf{0.40(0.47)} & \textbf{0.48(0.60)} & \textbf{2.23(0.28)} & \textbf{2.19(0.50)} & \textbf{0.23(0.23)} & \textbf{0.32(0.35)} \\ 
        DWR w/o w & 4.09(0.32) & 4.17(0.50) & 0.44(0.37) & 0.50(0.40) & 2.48(0.20) & 2.43(0.36) & 0.28(0.33) & 0.31(0.36) \\
        DWR w/o att & 6.96(0.44) & 7.33(0.54) & 0.83(0.90) & 0.73(0.85) & 4.60(0.55) & 5.12(0.61) & 0.34(0.41) & 0.52(0.54) \\ 
        DWR w/o att \& w & 7.05(0.61) & 7.03(0.66) & 0.98(0.63) & 0.96(0.70) & 4.52(0.69) & 5.02(0.51) & 0.41(0.50) & 0.49(0.66) \\  
        \hline \hline
        \multicolumn{9}{c}{Fb-pages-tvshow} \\ \hline
        & \multicolumn{2}{c}{$\sqrt{\epsilon_{PEHE}^{DE}}$}  &\multicolumn{2}{c}{$\epsilon_{MAE}^{DE}$} & \multicolumn{2}{c}{$\sqrt{\epsilon_{PEHE}^{SE}}$} & \multicolumn{2}{c}{$\epsilon_{MAE}^{SE}$} \\ \hline
        ~ & within-sample & out-of-sample & within-sample & out-of-sample & within-sample & out-of-sample & within-sample & out-of-sample\\ \hline
        CFR & 6.54(0.65) & 6.61(0.66) & 1.84(0.36) & 2.02(0.60) & N/A & N/A & N/A & N/A \\
        CFR+$\bar{z}$ & 6.44(0.57) & 6.55(0.63) & 1.94(0.44) & 1.93(0.59) & 6.61(0.36) & 6.73(0.43) & 0.86(0.12) & 1.03(0.36) \\ \hline
        TARNET & 6.89(1.13) & 6.96(1.16) & 2.18(0.52) & 2.08(0.57) & N/A & N/A & N/A & N/A \\
        TARNET+$\bar{z}$ & 6.49(0.55) & 6.48(0.39) & 1.92(0.37) & 1.81(0.59) & 6.61(0.36) & 6.65(0.52) & 0.91(0.10) & 0.86(0.23) \\ \hline
        ND & 8.27(0.59) & 8.38(0.70) & 1.32(0.22) & 1.13(0.53) & N/A & N/A & N/A & N/A \\
        ND+$\bar{z}$ & 8.15(0.54) & 8.23(0.62) & 1.35(0.33) & 1.15(0.52) & 6.72(0.35) & 6.66(0.26) & 1.34(0.26) & 1.38(0.26) \\ \hline 
        GIAL & 6.09(0.42) & 6.09(0.42) & 0.71(0.33) & 0.71(0.33) & N/A & N/A & N/A & N/A \\
        GIAL+$\bar{z}$ & 6.11(0.42) & 6.11(0.42) & 0.72(0.40) & 0.72(0.40) & 6.66(0.41) & 6.62(0.42) & 0.70(0.25) & 0.61(0.38) \\ \hline
        GCN$+\hat{H}SIC$ & 6.24(0.32) & 6.11(0.33) & 2.23(0.48) & 2.21(0.54) & 8.23(0.38) & 8.23(0.40) & 2.52(0.66) & 2.49(0.75) \\ 
        GPS &  8.32(0.35) & 8.16(0.39) & 7.08(0.52) & 6.98(0.56) & 8.32(0.35) & 8.16(0.39) & 1.89(0.09) & 1.98(0.33) \\
        PCN & 6.50(1.05) & 6.64(1.06) & 1.36(1.13) &
        1.51(1.14) & 5.50(1.05) & 5.67(1.05) & 0.87(0.06) & 0.88(0.28) \\ 
        NetEst & 9.03(0.67) & 9.29(0.59) & 2.46(0.84) & 2.72(0.91) & 6.60(0.38) & 6.70(0.52) & 0.51(0.34) & 0.45(0.43) \\ \hline 
        DWR & \textbf{5.61(0.28)} & \textbf{6.22(0.28)} & 0.64(0.53) & 0.73(0.59) & \textbf{6.72(0.46)} & \textbf{6.67(0.53)} & 0.74(0.40) & 0.61(0.59) \\ 
        DWR w/o w & 5.69(0.32) & 6.41(0.51) & 0.56(0.53) & 0.75(0.74) & 6.76(0.50) & 6.80(0.63) & 0.62(0.46) & 0.78(0.44) \\ 
        DWR w/o att & 6.05(0.40) & 6.48(0.45) & \textbf{0.43(0.53)} & \textbf{0.52(0.59)} & 6.81(0.55) & 6.86(0.41) & \textbf{0.47(0.52)} & \textbf{0.60(0.59)} \\ 
        DWR w/o att \& w & 5.71(0.29) & 6.22(0.33) & 0.45(0.37) & 0.40(0.50) & 6.78(0.43) & 6.86(0.68) & 0.50(0.51) & 0.65(0.61) \\ 
         \hline
    \end{tabular}
}}
% \vspace{-0.15in}
\end{table*}

\section{Empirical Results}
\subsection{Baselines}
We compared the proposed DWR algorithm with the following baselines: 
\emph{Counterfactual Regression (CFR).} CFR \citet{shalit2017estimating} tries to learn a balanced representation between treated and control groups by restricting the Wasserstein distance between the representation of the two groups. CFR is valid only when the SUTVA assumption holds and ignores the network structure.
\emph{TARNET.} TARNET \citet{shalit2017estimating} is an ablation version of CFR without the representation balancing regulation term.
\emph{Network Deconfounder (ND).} ND \citet{guo2020learning} utilizes the GCN to learn the representation of latent confounders and minimize the integral probability metric between treatment and control groups to obtain a balanced representation; 
\emph{Graph Infomax Adversarial Learning (GIAL).}
GIAL \citet{chu2021graph} further considers the imbalanced network structure on the basis of \citet{guo2020learning};
These methods ignore the interference in the networked observational data and cannot handle the complex confounding bias due to the presence of interference.
\emph{GPS.} \citet{forastiere2021identification} utilize the inherent balancing property of GPS and incorporate it into covariate adjustment methodologies to estimate treatment effects.
\emph{GCN$+\hat{H}$SIC.} \citet{ma2021causal} considers the treatment of neighboring nodes as a feature, thereby accounting for the influence of interference on the outcome. However, since only the treatment of neighbors is regarded as a feature, this approach cannot estimate the spillover effect.
\emph{PCN}. \citet{cristali2022using} leverage node embedding to represent the unmeasured confounders to estimate the spillover effect precisely. However, they still neglect the heterogeneous interference and cannot handle the complex confounding bias in the networked setting.
\emph{NetEst.} NetEst \citet{jiang2022estimating} considers the interference in the networked observational data and resolves the distribution mismatches through adversarial learning. However, it ignores the correlation between the treatment $T$ and the peer exposure $Z$. Moreover, it ignores the heterogeneous interference problem.
We also compared our method with the above methods with homogeneous peer exposure (the fraction of treated neighbors) into regression, denoted as CFR$+\bar{z}$, TARNET$+\bar{z}$, ND$+\bar{z}$, and GIAL$+\bar{z}$.
We show that these methods fail to estimate both direct and spillover effects without considering the confounding bias conducted by the interference and the heterogeneity in the networked observational dataset.
In addition to this, we also compared three ablation versions of the proposed DWR.
We use \textit{DWR w/o w} to denote the DWR algorithm without sample weights, \textit{DWR w/o att} to denote the DWR algorithm without attention weights, and \textit{DWR w/o att \& w} to show the performance of direct regression on the outcome.

\begin{figure}[tbp]
    \centering
    % \vspace{-0.15in}
    \subfigure[Average $\bar{z}$ v.s ground truth $z$]{
        \begin{minipage}[b]{.45\linewidth}
          \centering
          \includegraphics[width=\linewidth]{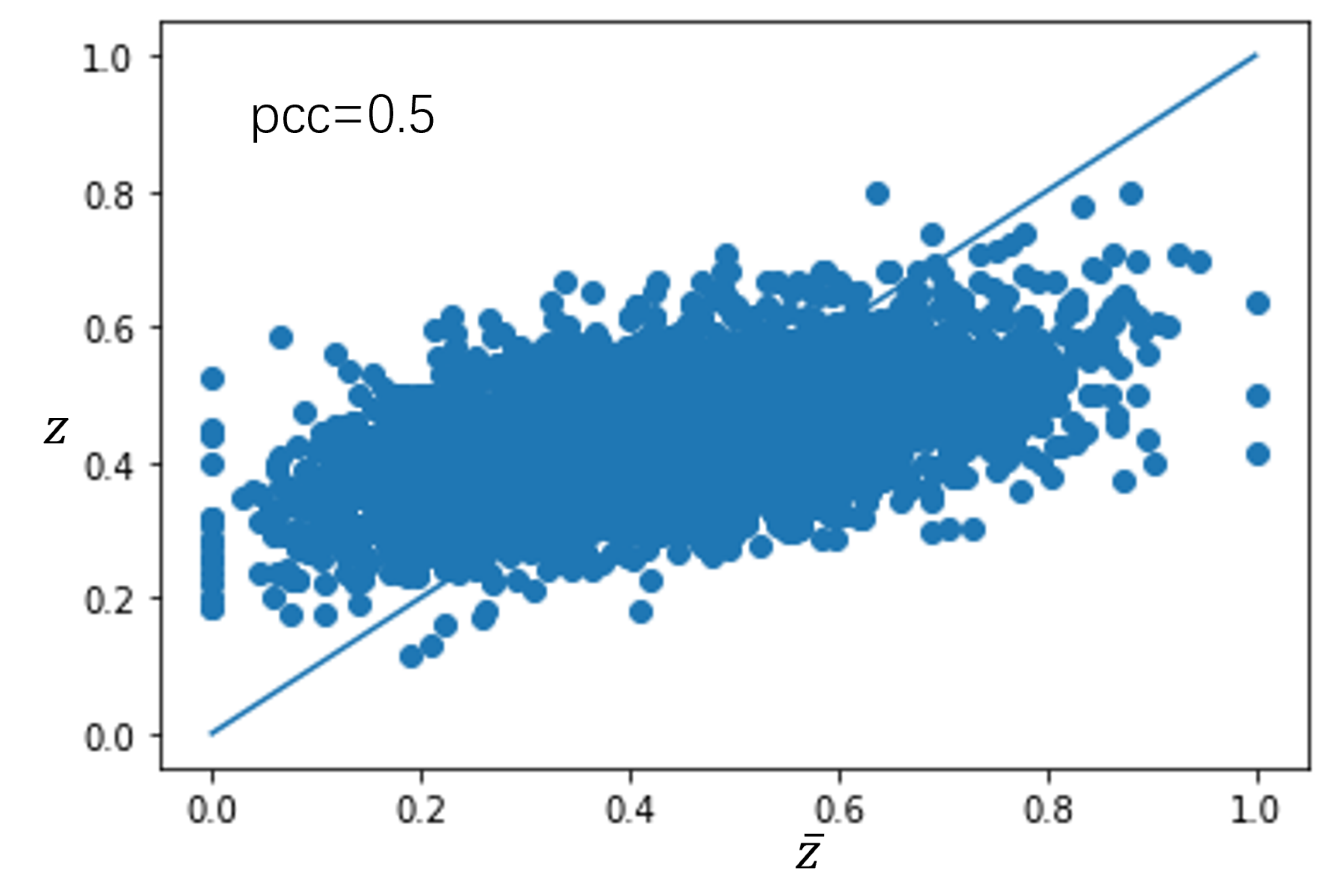}
            \vspace{-0.15in}
          \label{fig:average_z}
        \end{minipage}
  }\subfigure[Weighted average $\hat{z}$ v.s ground truth $z$]{
        \begin{minipage}[b]{.45\linewidth}
          \centering
          \includegraphics[width=\linewidth]{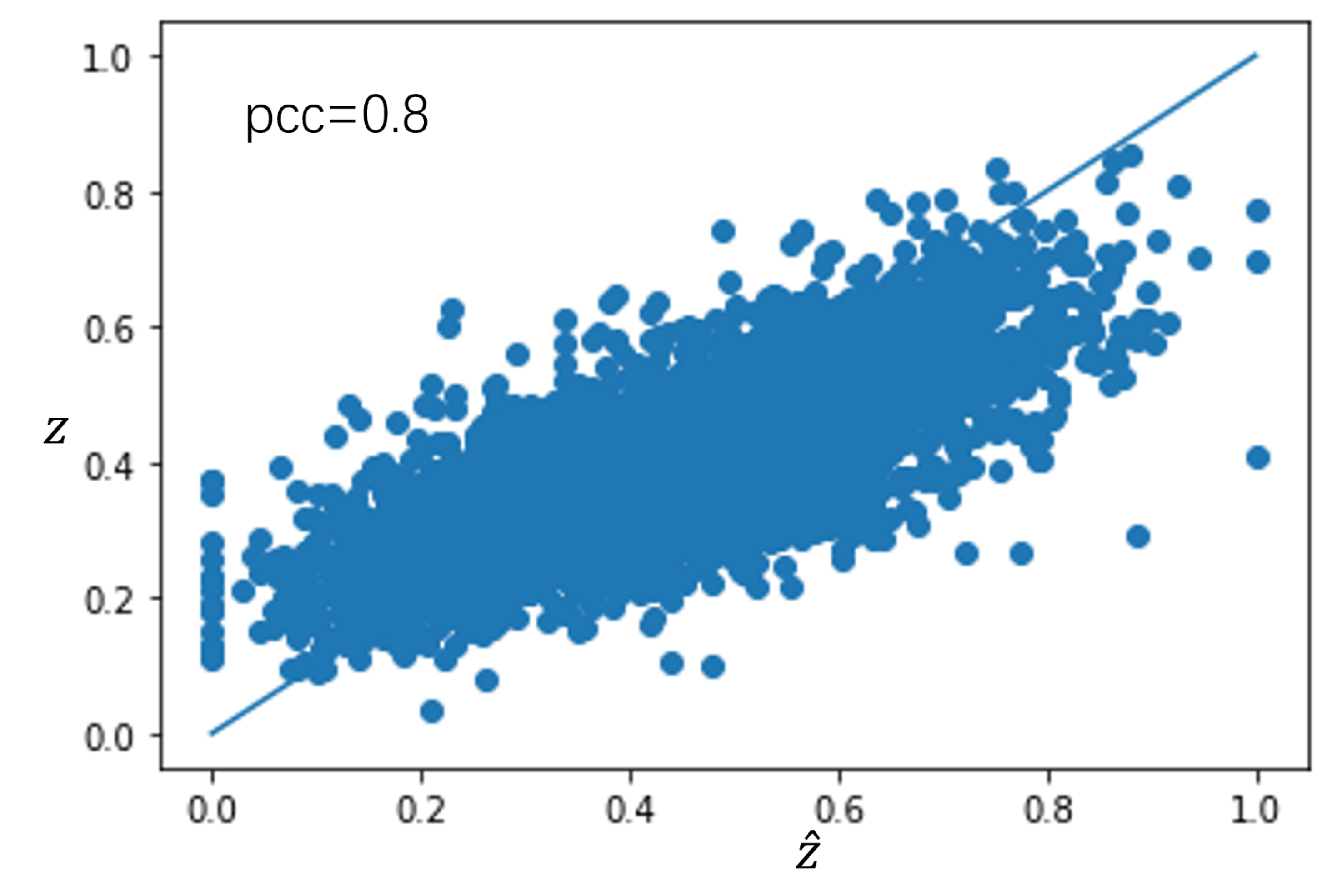}
            \vspace{-0.15in}
        \label{fig:attention_z}
        \end{minipage}
  }
  \vspace{-0.15in}
    \caption{The performance of the learned attention weights. Each blue dot represents a unit's peer exposure and the blue line $y=x$ represents the estimated peer exposure exactly equals the ground truth $z$ }
    \vspace{-0.15in}
    \label{fig:compare_z}
\end{figure}

\subsection{Dataset}
Following previous work on networked observational data \cite{guo2020learning}, we conduct experiments in the four real-world social networks:
\footnote{We use the preprocessed datasets of BlogCatalog and Flickr in \url{https://github.com/rguo12/network-deconfounder-wsdm20} and social networks \emph{Hamsterster} \& \emph{Fb-pages-tvshow} in \url{https://networkrepository.com/}}
\begin{itemize}
    \item \textbf{BlogCatalog \cite{leskovec2016snap}}. BlogCatalog is a graph dataset for a network of social relationships of bloggers listed on the BlogCatalog website.
    \item \textbf{Flickr \cite{leskovec2016snap}}. Flickr is an online social network where users share images and videos.
    The dataset is built by forming links between images sharing common metadata from Flickr. Edges are formed between images from the same location, submitted to the same gallery, group, or set, images sharing common tags, images taken by friends, etc.
    \item \textbf{Hamsterster} \cite{rossi2016interactive}. 
    The network is of the friendships and family links between users of the Hamsterster website.
    \item \textbf{Fb-pages-tvshow} \cite{rossi2016interactive}. These datasets represent blue-verified Facebook page networks of different categories. Nodes represent the pages, and edges are mutual likes among them.
\end{itemize}
Each node in the social network represents a user, and the edges between nodes represent their social relationships.
Due to the counterfactual problem, we can only observe one potential outcome, making it impossible to assess the performance of treatment effect estimates from real-world datasets. Here, we construct semi-synthetic datasets based on previous literature \cite{guo2020learning, shalit2017estimating}.
To maintain the homophily of the networked data, the covariates of each node are generated by the \textit{node2vec} algorithm \cite{grover2016node2vec}. We use the node embedding as covariates $x_i$ for each node $i$. Here we generate the embedding with dimension 10.
% We generated the covariates of each unit by $x_i \sim N(0, 1)$. 
To measure the heterogeneous interference among nodes, we use cosine distance as \cite{ma2011recommender} did to measure the similarity between connected units:
\begin{eqnarray}
    e_{ij} &=& \frac{x_i\cdot x_j}{\Vert x_i \Vert \Vert x_j \Vert} \\
    a_{ij} &=& \frac{exp(e_{ij})}{\sum_{k\in N(i)} exp(e_{ik})}
\end{eqnarray}
Then the peer exposure $z_i$ is generated by $\sum_{j\in N(i)} a_{ij} t_j$.
% The neighboring information of unit $i$ can be aggregated as $v_i=$.
Since we model the causal graph as a chain graph, with undirected edges between the treatment between linked units, we do not have direct access to the joint distribution of the $P({T}|{X})$. 
Following \cite{tchetgen2021auto}, we use the conditional density $p(t_i|x_i,X_{N(i)},{T}_{N(i)})$ (Gibbs factors) to generate the data by Gibbs sampling, which simulates the joint distribution by the following conditional densities:
\begin{eqnarray}
    p(t_i|x_i,X_{N(i)},{T}_{N(i)}) = \alpha_0 x_i + \alpha_1 \sum_{j\in N(i)} a_{ij} x_j + \alpha_2 z_i
\end{eqnarray}
where $\alpha_0$, $\alpha_1$ and $\alpha_2$ are generated by $N(0,1)$.
By Gibbs sampling algorithm, we can generate a stationary distribution of $P(\mathbf{T}|\mathbf{X})$.

The outcome $y_i$ is generated in the following way:
\begin{eqnarray}
\label{eq:y_generation}
    y_i = T \left(\beta_0 x_i + \beta_0 \sum_{j\in N(i)} a_{ij} x_j + \beta_2 z_i \right) \nonumber \\
    + (1-T) \left(\beta_1 x_i + \beta_1 \sum_{j\in N(i)} a_{ij} x_j + \beta_2 z_i \right)
\end{eqnarray}
where $\beta_0 \sim unif(1,2)$ and $\beta_1, \beta_2 \sim unif(0,1)$. With the generated data, we verify the effectiveness of our method compared with the baselines in the next section.

\subsection{Results}
In this section, we report the results of treatment effects estimation on four datasets. We consider the Precision in Estimation of Heterogeneous Effect $\sqrt{\epsilon_{PEHE}} = \sqrt{\frac{1}{n}\sum_{i=1}^n(\hat{\tau}(\mathbf{X}_i) - \tau (\mathbf{X}_i))}$ for estimating individual treatment effects. Here we consider both direct and spillover effects and denote them as $\sqrt{\epsilon_{PEHE}^{DE}}$ and $\sqrt{\epsilon_{PEHE}^{SE}}$. Besides, we consider the mean absolute error of the average treatment effect estimation for both direct and spillover effects, denoted as $\epsilon_{MAE}^{DE}$ and $\epsilon_{MAE}^{SE}$. We carry out the experiments 10 repetitions independently and report the standard deviation of the considered estimate.

The main results are shown in Tab.\ref{tab:main_res}. From the results, we can draw the following conclusions:
(1) Traditional methods neglect the interference problem in networked observational data, leading to poor performance estimating treatment effects.
(2) Although we add $\bar{z}$, which is the proportion of treated neighbors, into the regression process, these models still fail to estimate both direct and spillover effects because of the complex confounding bias and the heterogeneous problem in the networked observational data. 
(3) NetEst considers resolving the bias from the association between $Z$ and $\mathbf{X}$ and achieving better performance than the other baselines.
Since it ignores the association between $T$ and $Z$, meanwhile neglecting the heterogeneous interference, its performance is still worse than that of DWR.
(4) Our method outperforms the baseline methods in both direct and spillover treatment effects estimation since we solve two challenges of heterogeneous and complex confounding bias.
(5) The ablation study shows that the attention mechanism and the sample-reweighting schema significantly improve the model's performance.
(6) From the experimental results, we can observe that the mechanisms of the two weights have a mutually reinforcing effect. This is because if we use Sample weights to reduce confounding bias, the learned attention weights will be more accurate. At the same time, 
the learned attention weights enable the representation and peer exposures to correctly capture the heterogeneity in the network, making the learned sample weights better eliminate the confounding bias.
% (5) In estimating treatment effects, Baseline's model is significantly biased since it ignores potential heterogeneity.

\textbf{Results on attention weights learning.} 
We investigate the performance of attention weights learning in Fig.\ref{fig:compare_z}.
From these results, we have the following observations and analyses:
(1) Without considering the problem of heterogeneous interference, simply using the proportion of treated neighbors $\bar{z}$ as the peer exposure is severely biased from the ground truth $z$.
(2) With the attention weights learned to measure the heterogeneity in the network, $\hat{z}$ is unbiased compared with the ground truth $z$.
(3) The Pearson correlation coefficient between estimated peer exposure $\hat{z}$ and ground truth $z$ increases from 0.5 to 0.8, which means that the heterogeneity is captured by the learned attention weights.

\begin{figure*}[htbp]
% \begin{wrapfigure}{r}{0.7\linewidth}
    \centering
    \subfigure[Flickr]{
        \begin{minipage}[b]{.5\linewidth}
          \centering
          \includegraphics[width=\linewidth]{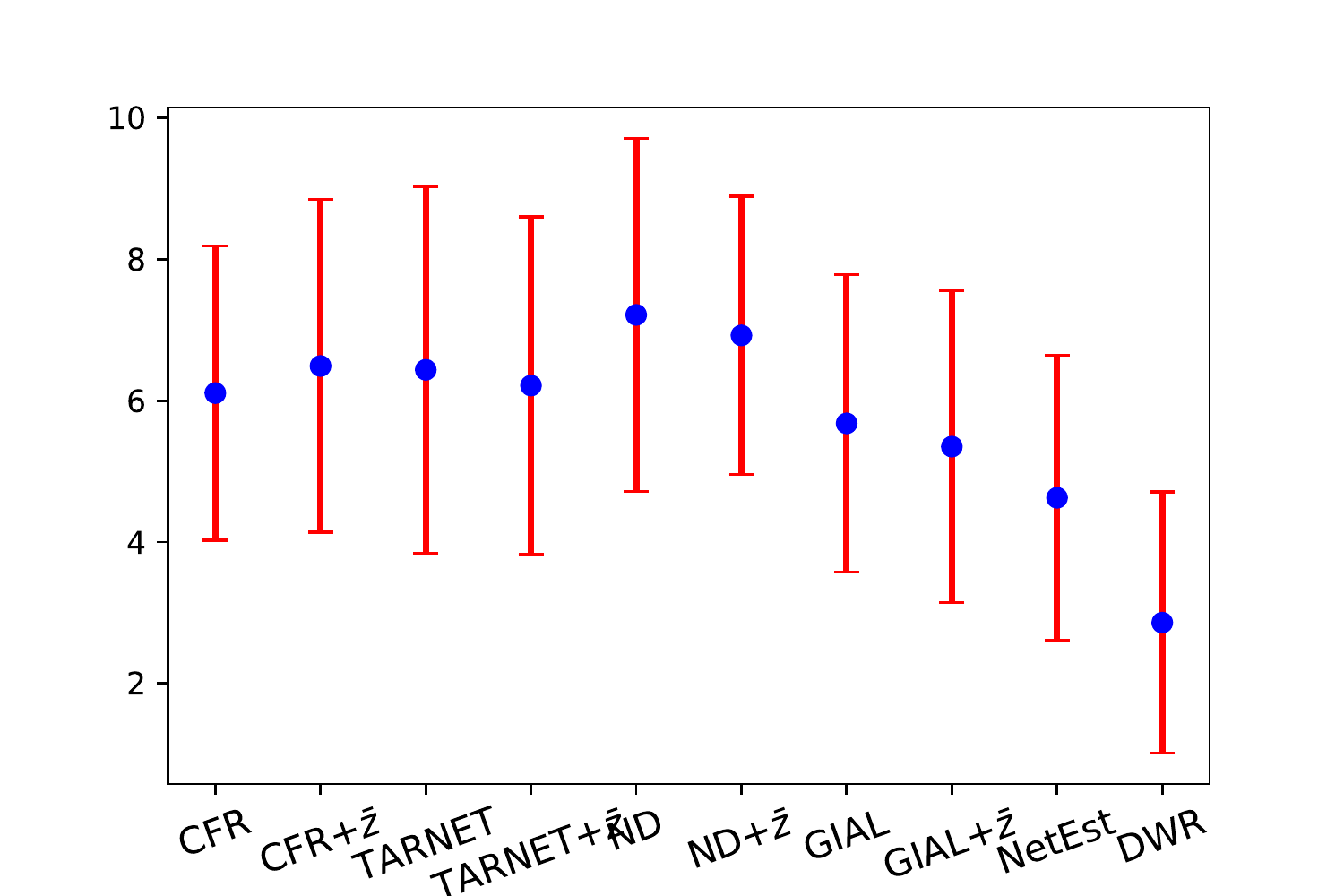}
        \end{minipage}
  }\subfigure[BlogCatalog]{
        \begin{minipage}[b]{.5\linewidth}
          \centering
          \includegraphics[width=\linewidth]{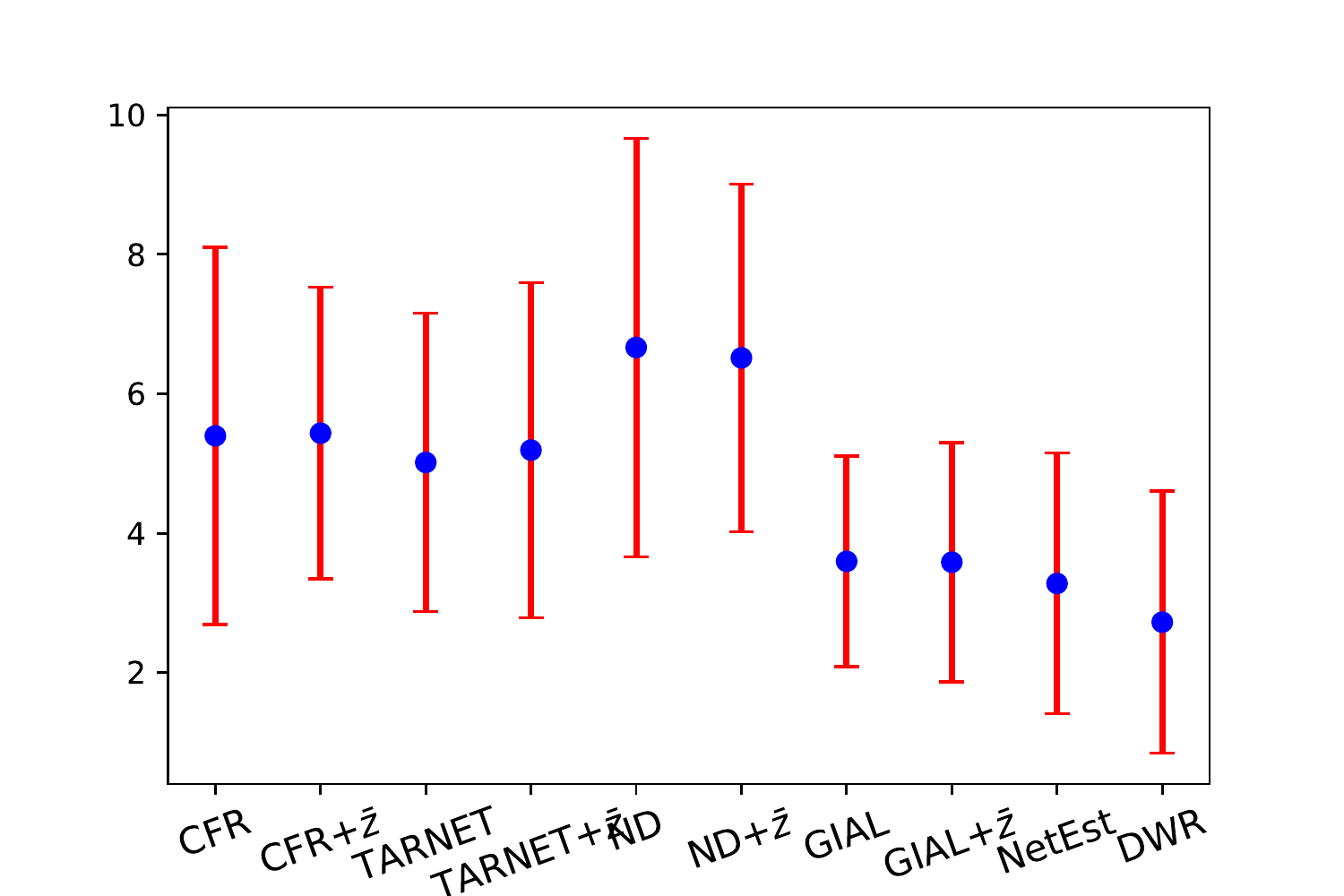}
        \end{minipage}
  }
  
  \subfigure[Hamsterster]{
        \begin{minipage}[b]{.5\linewidth}
          \centering
          \includegraphics[width=\linewidth]{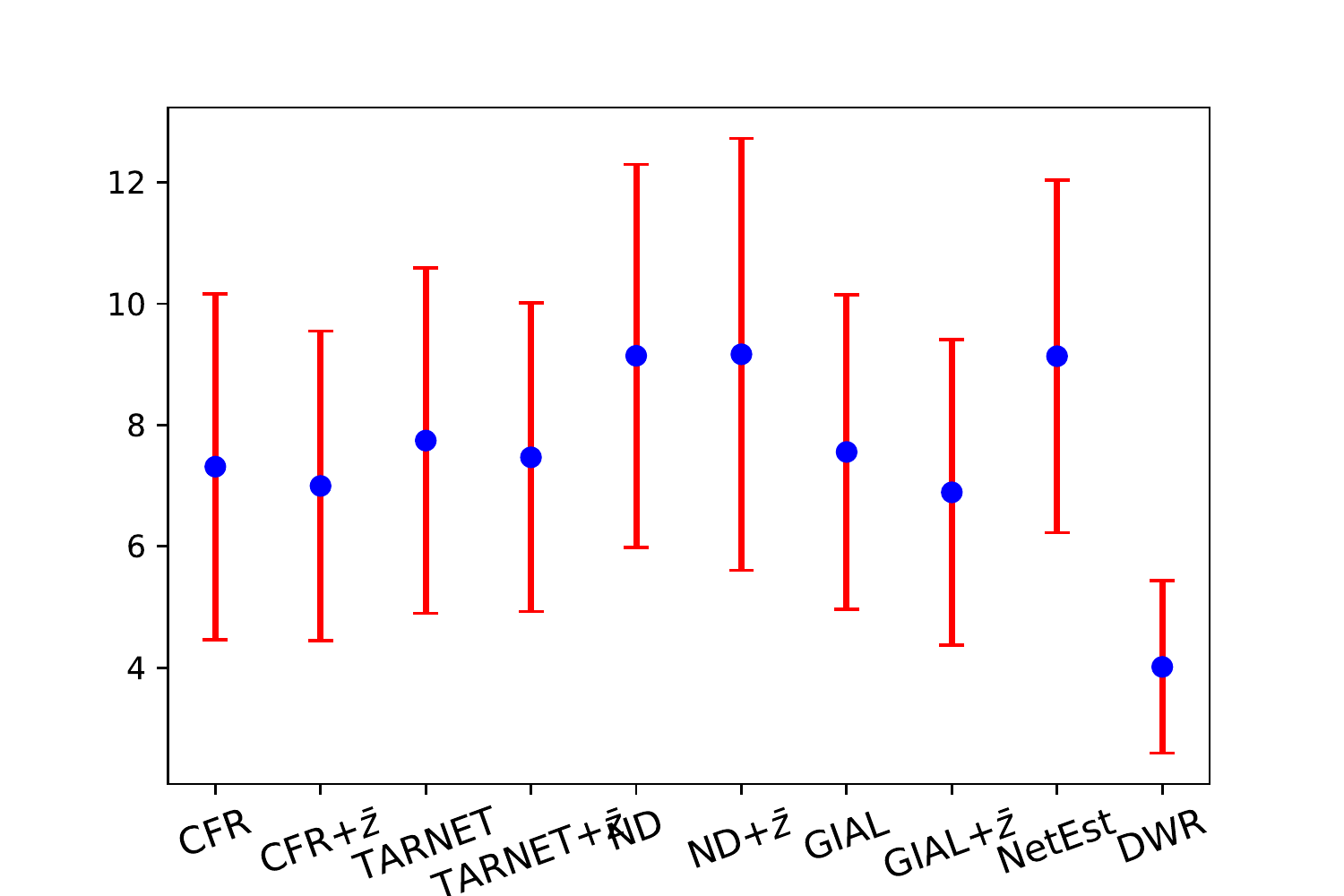}
        \end{minipage}
  }\subfigure[Fb-pages-tvshow]{
        \begin{minipage}[b]{.5\linewidth}
          \centering
          \includegraphics[width=\linewidth]{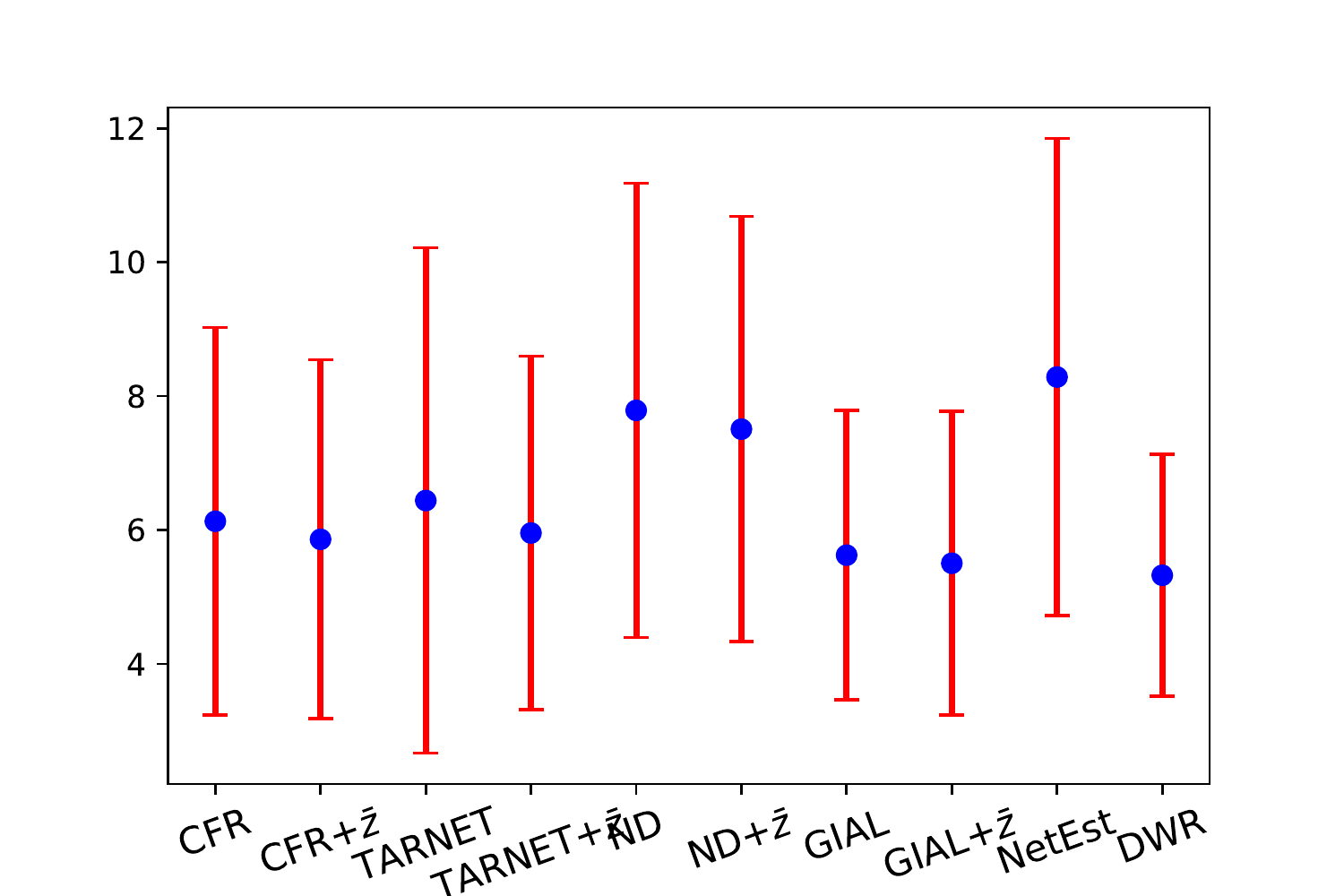}
        \end{minipage}
  }
    % \vspace{-0.15in}
    \caption{Performance on counterfactual prediction in four datasets.}
    %   \vspace{-0.15in}
    \label{fig:cfp}
% \end{wrapfigure}
\end{figure*}

\textbf{Results on counterfactual prediction.}
We investigate the performance of counterfactual prediction on these datasets. Here, we generated the dataset with $t \sim bern(0.5)$ and $z \sim unif(0,1)$. We show the performance of RMSE and the standard deviation of the estimation in Fig.\ref{fig:cfp}.
From the experiment results, we have the following observations and analyses:
(1) Traditional methods fail to predict counterfactual outcomes since they ignore the interference in the network data and cannot handle the complex confounding bias induced by the interference.
(2) Adding $\bar{z}$ to the regression still fails to address the challenges of heterogeneity and confounding bias.
(3) The proposed DWR algorithm effectively addresses heterogeneity and confounding bias challenges and achieves good performance on counterfactual prediction.

\begin{figure}[tbp]
    \centering
    % \vspace{-0.15in}
    \subfigure[Direct treatment effect estimation]{
        \begin{minipage}[b]{.5\linewidth}
          \centering
          \includegraphics[width=\linewidth]{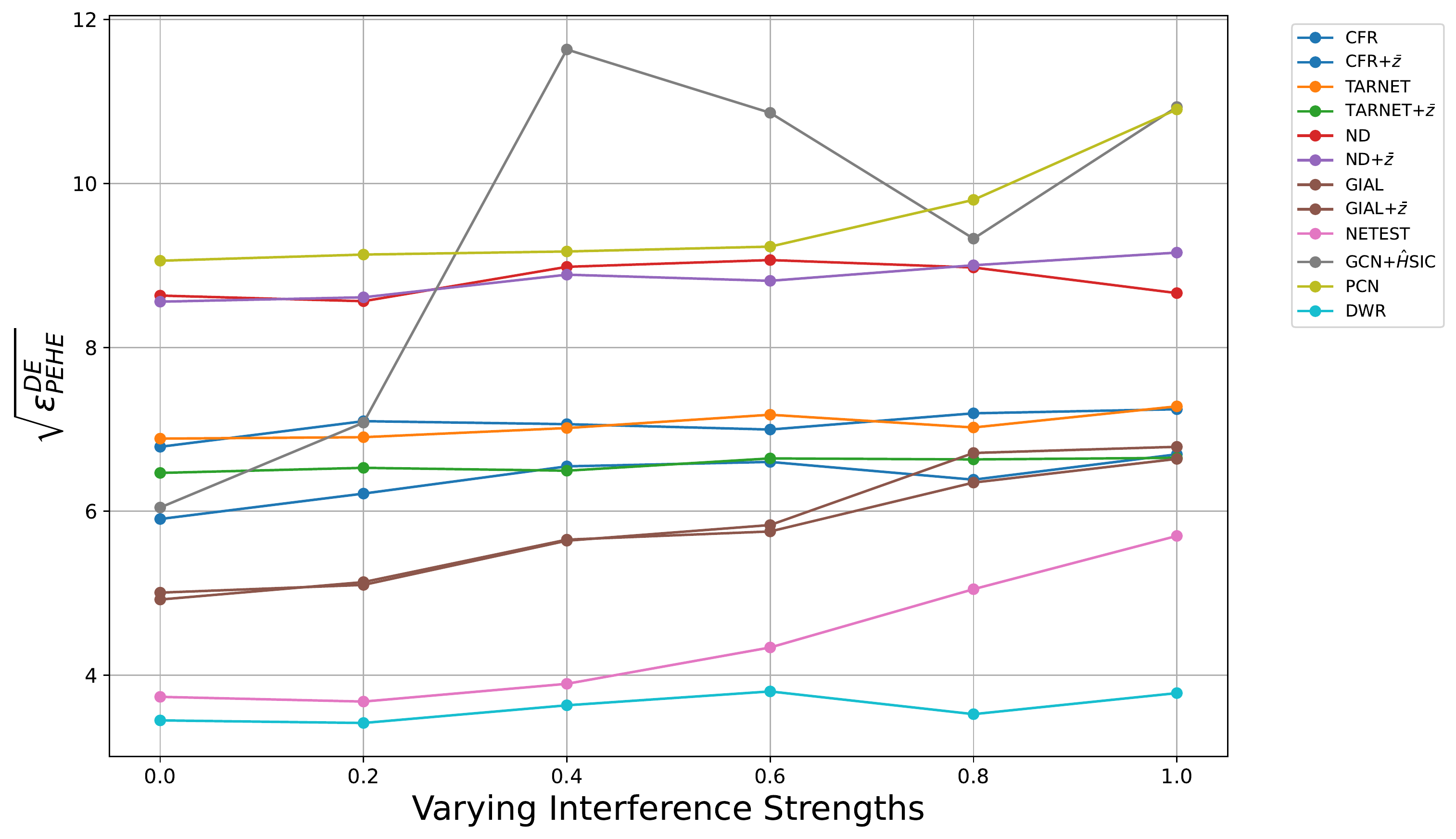}
            \vspace{-0.15in}
        \end{minipage}
  }\subfigure[Spillover effect estimation]{
        \begin{minipage}[b]{.5\linewidth}
          \centering
          \includegraphics[width=\linewidth]{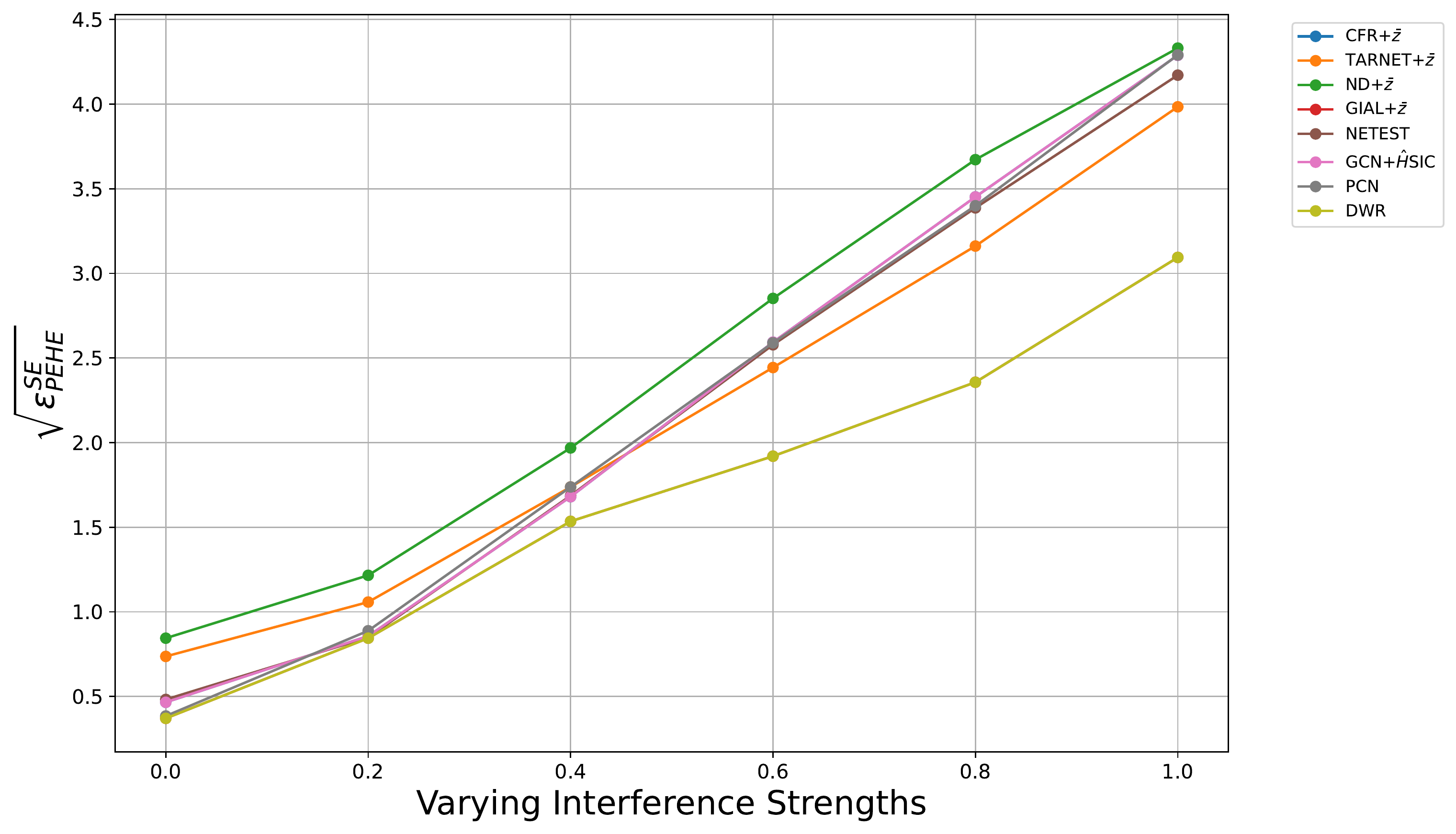}
            \vspace{-0.15in}
        \end{minipage}
  }
  \vspace{-0.15in}
    \caption{Direct treatment effect and spillover effect estimation varying interference strengths in Flickr dataset.}
    \vspace{-0.15in}
    \label{fig:varing_interference}
\end{figure}

\begin{figure*}[htbp]
    \centering
    \subfigure[Flickr]{
        \begin{minipage}[b]{.4\linewidth}
          \centering
          \includegraphics[width=\linewidth]{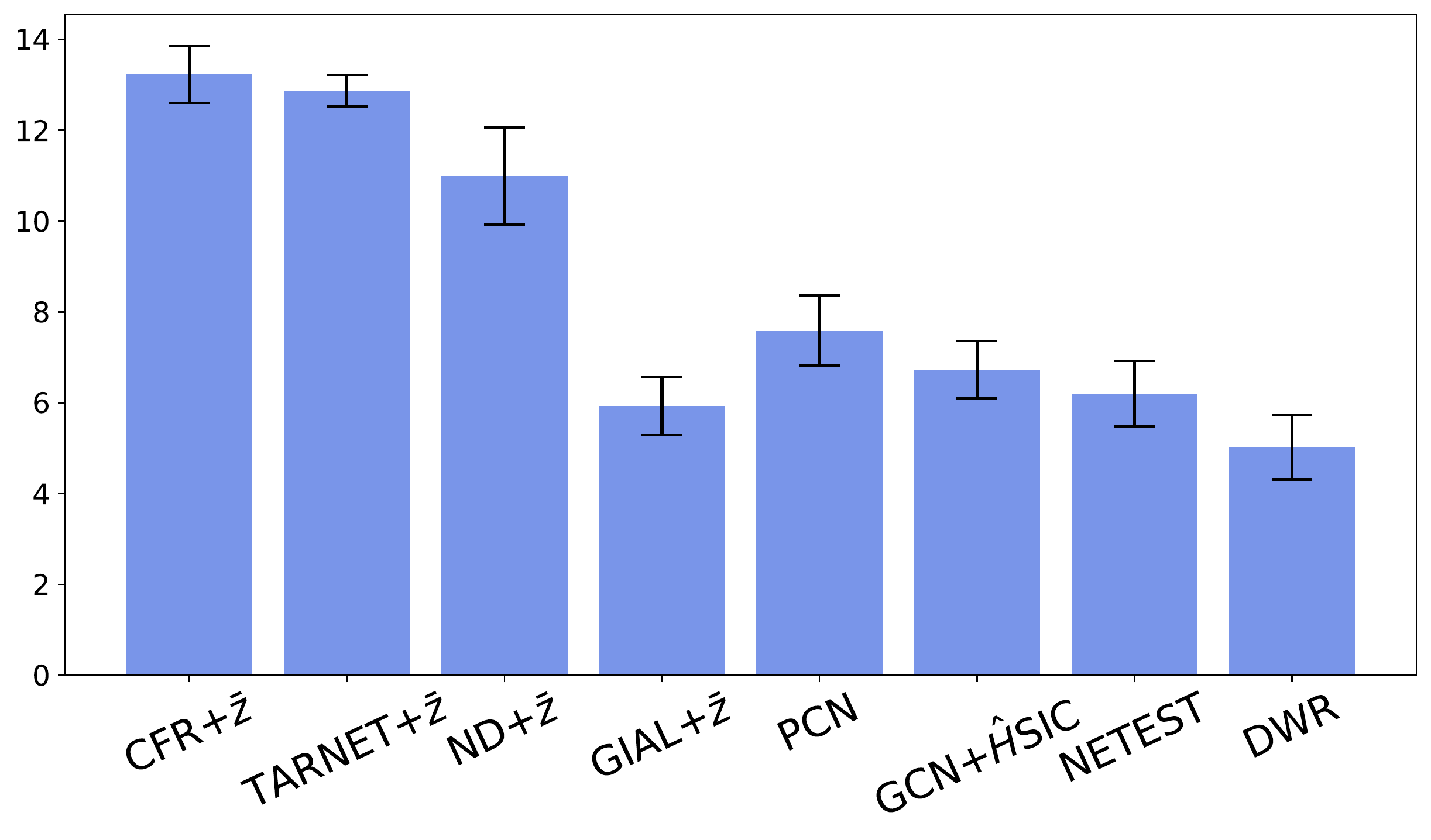}
        \end{minipage}
  }\subfigure[BlogCatalog]{
        \begin{minipage}[b]{.4\linewidth}
          \centering
          \includegraphics[width=\linewidth]{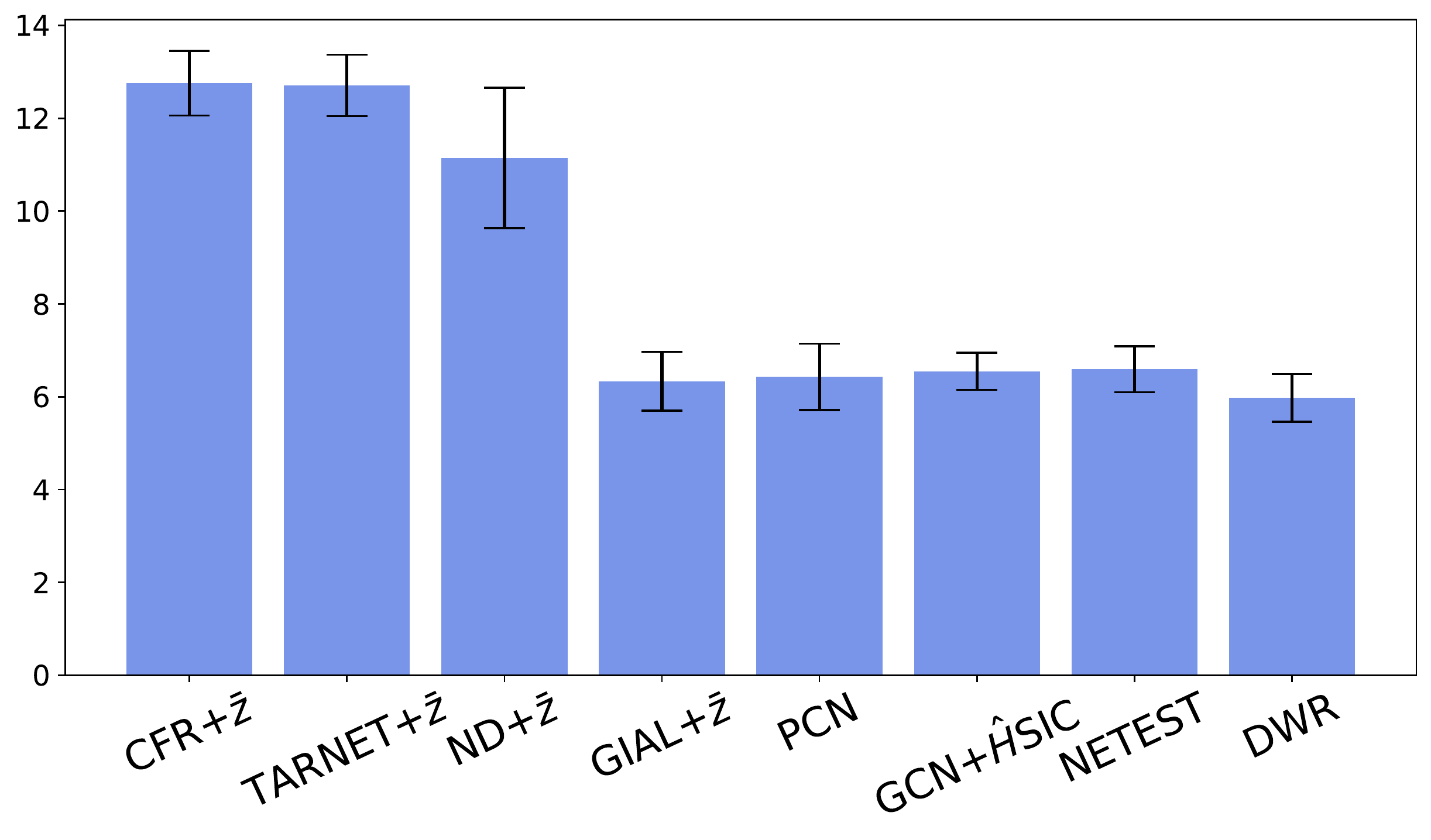}
        \end{minipage}
  }
  
  \subfigure[Hamsterster]{
        \begin{minipage}[b]{.4\linewidth}
          \centering
          \includegraphics[width=\linewidth]{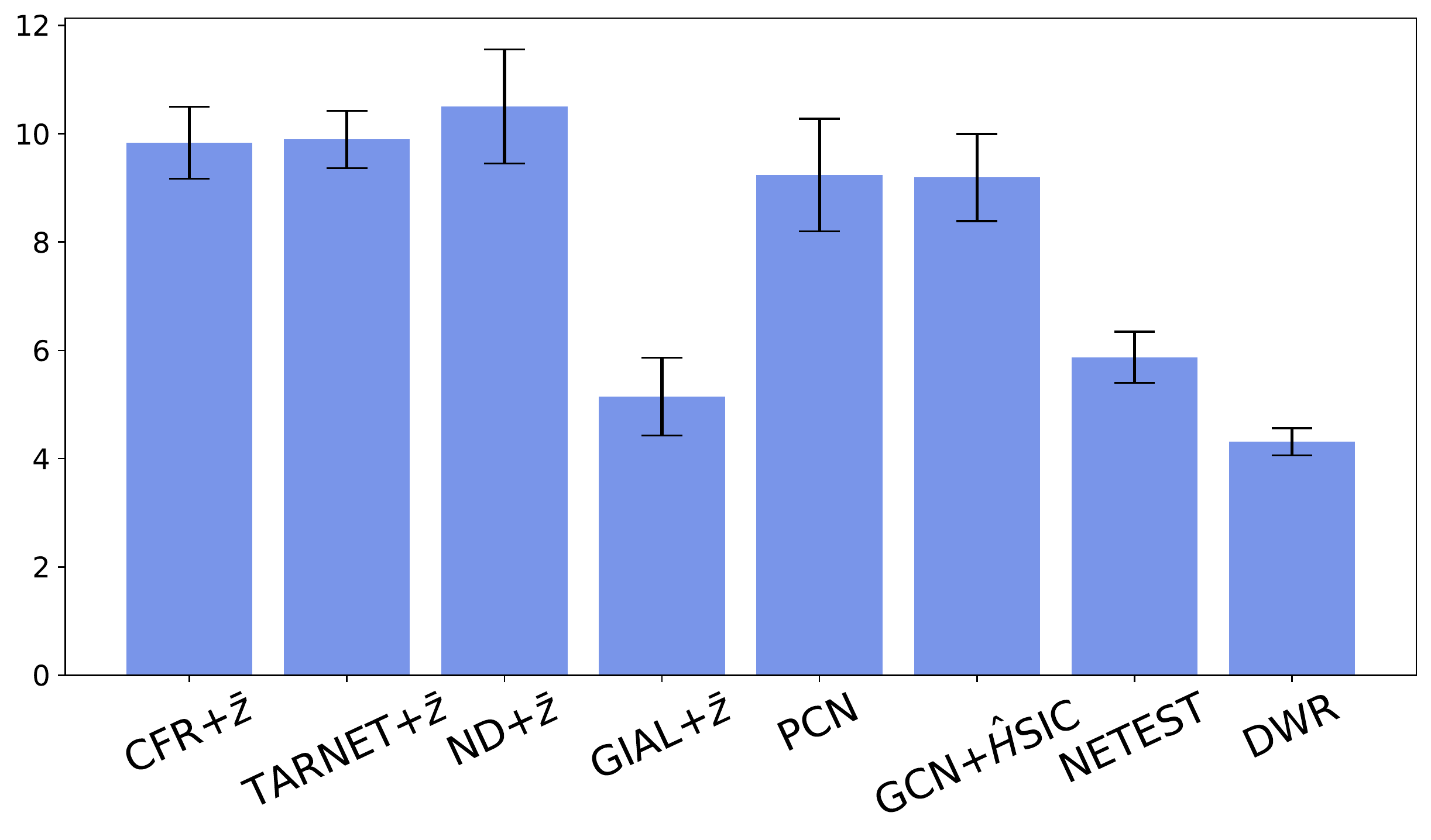}
        \end{minipage}
  }\subfigure[Fb-pages-tvshow]{
        \begin{minipage}[b]{.4\linewidth}
          \centering
          \includegraphics[width=\linewidth]{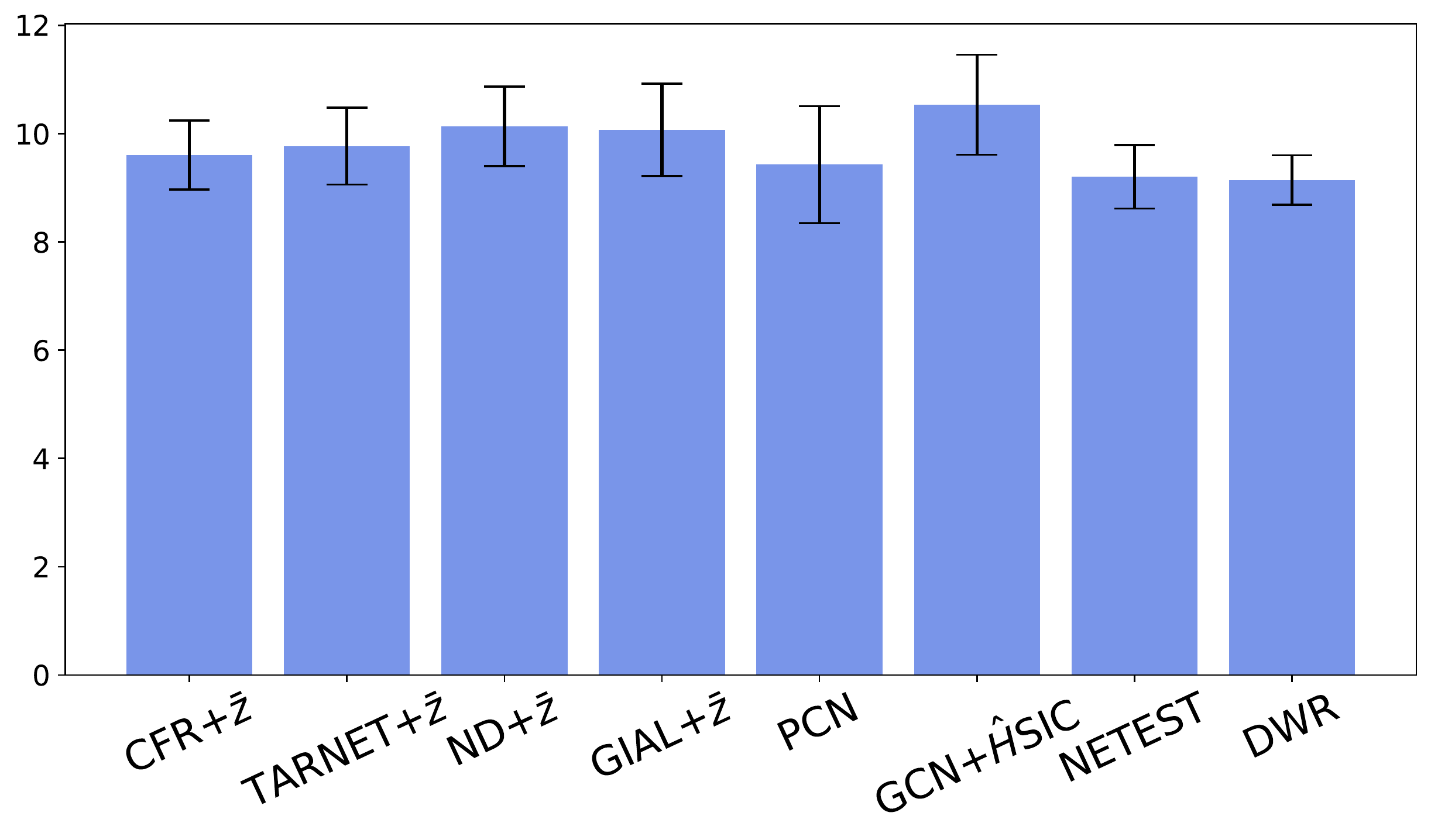}
        \end{minipage}
  }
    % \vspace{-0.15in}
    \caption{Performance on the PEHE of total effect estimation in four datasets.}
    %   \vspace{-0.15in}
    \label{fig:total_effect}
% \end{wrapfigure}
\end{figure*}

\textbf{Varying the Interference strength.} 
We investigated the effectiveness of various methods in estimating direct and spillover effects under different levels of interference, as defined in Equation \ref{eq:y_generation}, with the results shown in Fig.\ref{fig:varing_interference}. The analysis indicates that an increase in the strength of interference typically results in a decrease in model performance for estimating both direct and indirect causal effects. Notably, the estimation performance for spillover effects declines more as the strength of interference increases, compared to the estimation of direct causal effects. Remarkably, our proposed DWR method consistently outperforms other models in all tested scenarios. Furthermore, the DWR method demonstrates stable performance across various degrees of interference.

\textbf{Total Effect Estimation.} 
Besides the direct causal effect and spillover effect, here we also examine the accuracy of different methods in assessing the total effect. The total effect, which combines the direct causal effect and spillover effect, is defined as follows: $\mathbbm{E}[y(t=1,z=1)-y(t=0,z=0)|\mathbf{X}]$. In Fig.\ref{fig:total_effect}, we show the performance of different methods in estimating the total effect. From the results, we can conclude that: (1) For methods that consider network interference, their performance in estimating the total effect is superior to that of methods that do not consider interference; (2) The proposed DWR method, due to its resolution of heterogeneity and complex confounding bias issues, performs best in estimating the total effect.

\textbf{Implementation details.} 
Our model architecture comprises two Multilayer Perceptron (MLP) layers with hidden states of {32, 64}, followed by a Graph Convolutional Network (GCN) layer for feature extraction. For the estimation of $\hat{Y}_0$ and $\hat{Y}1$, we implement two distinct heads, each featuring a set of three MLP layers with {128, 128, 128} hidden states. The Adam optimizer is employed for minimizing the loss function detailed in Eq.\ref{eq:regression}. For the purpose of learning sample weights, a three-layer MLP with {64, 64, 64} hidden states is utilized, optimizing the $\mathcal{L}{\pi}$ loss function as described in Eq.\ref{eq:loss_weights} using stochastic gradient descent. We incorporate ReLU as the activation function between hidden layers and apply a dropout rate of 0.5 to mitigate overfitting.

To implement the baseline methods, we adhered to the guidelines from Guo et al. \cite{guo2020learning} for Counterfactual Regression (CFR), TARNET, and Network Deconfounder (ND), optimizing the parameter $\alpha$ through a search in the set ${1e-1, 1e-2, 1e-3, 1e-4, 1e-5}$. For the Graph Inference and Learning (GIAL) method, we followed Chu et al.'s approach \cite{chu2021graph}, keeping the hyperparameters $\alpha$ and $\beta$ at $1e-3$. This decision was based on GIAL’s stable performance across various parameters, selecting the optimal combination as reported in their original paper\footnote{Given GIAL's performance stability over a wide parameter range, we opted for the best parameter combination cited in the publication.}. In the case of Network Estimator (NetEst), as per Jiang et al. \cite{jiang2022estimating}, we set the hyperparameters $\alpha$ and $\gamma$ to 0.5, aligning with their recommendations. For GCN$+\hat{H}$SIC, based on Ma et al. \cite{ma2021causal}, we conducted a search for the optimal $\kappa$ within ${0.001, 0.005, 0.1, 0.2}$.

\section{Related works}
\textbf{Networked observational data.}
\citet{guo2020learning} proposed net-deconfounder to utilize the auxiliary network connection information to learn better representation for counterfactual prediction.
\citet{chu2021graph} further considered the imbalance in the network structure and proposed Graph Infomax Adversarial Learning framework.
\citet{veitch2019using} viewed the network structure as a proxy of unobserved confounders and proposes to learn node embedding to adjust unmeasured confounding partially.
These methods try to relax the unconfoundedness assumption by exploring the network information. We posit a similar networked unconfoundedness assumption to ensure the identification of the treatment effect.
Besides, these methods assume the SUTVA assumption to be held, which makes these methods limited in estimating treatment effects in the presence of interference.

\textbf{Interference.}
Much of the current work in the area of causal inference has focused on the problem of network interference in observed data.
\citet{liu2016inverse} proposed a generalized inverse probability-weighted estimator and two Hájek-type stabilized weighted estimators to estimate both direct and spillover treatment effects.
Similarly, to estimate treatment and spillover effects, \citet{van2014causal} and \citet{sofrygin2017semi} proposed targeted maximum likelihood estimators.
\citet{ogburn2017causal} extended this TMLE estimator to allow for contagion and homophily dependence. 
\citet{aronow2017estimating} presents a comprehensive framework for estimating average causal effects under conditions of interference between units in a social network experiment. \citet{bhattacharya2020causal} propose to estimate causal effects in situations where there is significant uncertainty about the network structure of data dependence.
% Through asymptotic results, they increased the number of times per node as the network grew.
\citet{tchetgen2021auto} takes into account the effect of treatment between neighboring nodes by modeling the traditional causal graph from a directed acyclic graph to a chain graph.
% and obtains the joint distribution of intervening variables from the edge distribution by Gibbs sampling. 
These traditional methods focus only on estimating group effects and fail to make predictions about individual effects. Moreover, these methods ignore the problem of heterogeneous interference. 

Some recent work on estimating individual treatment effects considers modeling the interference. \citet{ma2021causal} integrate neighboring treatments as features to enhance the prediction model.
\citet{ma2022learning} proposed to model the high-order interference with a hypergraph neural network.
These methods focus on eliminating the bias to precisely estimate the direct effect and lack the ability to estimate the peer effect.
\citet{forastiere2021identification} incorporate generalized propensity score into the covariate adjustment methods to estimate treatment effects.
\citet{jiang2022estimating} proposed NetEst to resolve the bias from interference and estimate both direct and indirect effects.
\citet{cristali2022using} leverage node embedding to resolve the unmeasured confounding problem in the presence of interference.
However, they neglected to model the association between treatment and peer exposure. Moreover, they also overlooked the heterogeneous interference in the social network. 
In this paper, we focus on estimating individual treatment effects from networked observational data in the presence of interference.

\textbf{Graph Attention Networks.}
Traditional Graph Convolutional Networks (GCNs) \cite{kipf2016semi} for learning node representations. They treat each node equally without differentiating the neighboring nodes.
\cite{velivckovic2017graph} proposed the Graph Attention Network to address this problem by weighting the contribution of the neighboring nodes differently using an attention mechanism.
% \cite{song2019session} leverage the graph attention network to model the context-dependent social influence, which has a similar motivation as we do.
We adopt the concept of learning the contribution of neighboring nodes using an attention mechanism to relax the anonymous interference assumption and solve the heterogeneous challenge.

\textbf{Learning Sample Weights}
\cite{li2020continuous} propose to learn sample weights to decorrelate the continuous treatment and the confounders.
\cite{zou2020counterfactual} use the learned sample weight into regression to resolve the confounding bias conducted by bundle treatment.
Here we propose to learn sample weights to resolve the confounding bias in the network scenario. As discussed before, the presence of interference in the network scenario complicates the problem. For these methods, they cannot handle the confounding bias caused by the correlation between these three components.

\section{Conclusion}
In this paper, we focus on estimating individual treatment effects from networked observational data in the presence of heterogeneous interference.
We summarize two challenges under this problem: the heterogeneity of interference and the complex confounding bias in networks.
To address these challenges, we propose a novel Dual Weighting Regression (DWR) algorithm by learning weights to differentiate the neighboring nodes with the attention mechanism and learning sample weights to resolve the confounding bias.
The learning process of the proposed DWR algorithm can be formulated as a bi-level optimization problem.
We further give the generalization error bound of the treatment effect estimation to show the effectiveness of the proposed methods. 
Extensive experimental evaluations demonstrate the superiority of our algorithm over the other baselines.

Looking forward, incorporating self-supervised learning into the DWR framework presents an exciting research avenue. This approach could enable the extraction of latent embeddings from heterogeneous graphs, potentially addressing the critical challenge of unmeasured confounders. Our current work, bound by standard assumptions, has yet to tackle this issue, leaving a significant gap for future exploration through self-supervised learning techniques in heterogeneous graphs\cite{jing2022x, chen2023heterogeneous}.

% \input{sample-authordraft.bbl}
\bibliographystyle{ACM-Reference-Format}
\bibliography{sample-base}

\end{document}